\documentclass{article}
\usepackage{preprint,times}

\usepackage{amsmath,amsfonts,bm}

\def\eqref#1{equation~\ref{#1}}

\def\1{\bm{1}}

\DeclareMathAlphabet{\mathsfit}{\encodingdefault}{\sfdefault}{m}{sl}
\SetMathAlphabet{\mathsfit}{bold}{\encodingdefault}{\sfdefault}{bx}{n}

\newcommand{\softmax}{\mathrm{softmax}}

\DeclareMathOperator*{\argmax}{arg\,max}

\usepackage{amssymb}
\usepackage{booktabs}
\usepackage{multirow}
\usepackage{microtype}
\usepackage{xcolor}
\usepackage{algorithm}
\usepackage{algpseudocode}
\usepackage{array}
\usepackage{graphicx}
\usepackage{float}
\usepackage{booktabs,tabularx,ragged2e}
\usepackage{placeins}
\usepackage{hyperref}
\usepackage{url}
\usepackage{wrapfig}
\newcolumntype{L}[1]{>{\raggedright\arraybackslash}p{#1}}

\newcommand{\gpmap}{\textsc{GP-MAP}}
\newcommand{\gpfb}{\textsc{GP-FB}}
\newcommand{\pfn}{\textsc{PFNs4BO}}
\newcommand{\pfnhebo}{\textsc{PFNs4BO-HEBO}}
\newcommand{\pfnbnn}{\textsc{PFNs4BO-BNN}}

\newcommand{\fibo}{\textsc{FIBO}}
\newcommand{\ours}{\textsc{IQS-BO}}
\newcommand{\oursgp}{\textsc{IQS-BO-GP}}
\newcommand{\ourspfn}{\textsc{IQS-BO-PFN}}
\newcommand{\ctrl}{\textsc{EI-CAND}}

\usepackage{amsthm,amsfonts,amsmath}

\newtheorem{proposition}{Proposition}

\theoremstyle{plain}

\title{IQS-BO: In-Context Query Selection for Bayesian Optimisation}

\author{Luca Geminiani$^{1}$ \quad Nadja Klein$^{1}$ \\
\normalfont $^{1}$Scientific Computing Center (SCC), Karlsruhe Institute of Technology (KIT)}

\begin{document}
\maketitle

\begin{abstract}
Bayesian Optimisation (BO) is a powerful framework for the optimisation of expensive black-box functions, but typically requires refitting a surrogate and maximising an acquisition function at every evaluation step. In-context approaches based on Prior-data Fitted Networks (PFNs) amortise part of this cost by pre-training transformers on functions drawn from synthetic priors. \pfn{} amortises the surrogate but still relies on a numerically maximised acquisition function, while \fibo{} performs BO fully in-context by sampling optimiser locations from a learned density, which fixes the decision rule and admits no surrogate. Learned acquisition functions score a finite candidate set with a trained network, but, lacking a label for the query, learn the score by reinforcement learning on previously solved tasks. We propose \ours, a PFN that learns the query decision by supervised learning on synthetic priors. In a single forward pass, \ours{} predicts the probability that each candidate maximises the objective over the set, and we show that the minimiser of its objective is the posterior probability of this event. The model can be pre-trained without a surrogate for fully in-context BO, or take the predictions of a fixed probabilistic surrogate as additional input, amortising only the decision step. Our method proposes queries at a fraction of the cost of acquisition-based methods, while either matching or outperforming standard BO with Gaussian processes (GPs) and available in-context methods on synthetic and real-world benchmarks. Finally, we propose a mixture prior for pre-training PFNs which combines samples from GPs with functions exhibiting warped inputs, isolated narrow optima, or plateaus that are poorly modeled by stationary kernels common in GP surrogates. We show that pre-training on this prior can lead to  improved optimisation performance. 

\end{abstract}

\section{Introduction}
\label{sec:intro}
Bayesian optimisation (BO) \citep{jones1998efficient,shahriari2016taking,garnett2023bayesopt} aims to maximise an expensive black-box function, and has proven particularly beneficial when objective evaluations are costly and only a limited number of evaluations is available. Given an initial dataset, BO typically alternates between fitting a probabilistic surrogate and selecting the next evaluation point by maximising an acquisition function. Both steps are repeated for every new observation: the surrogate must adapt as the dataset grows, while the acquisition function is a non-convex objective whose optimisation often requires dedicated sampling, restart, and numerical optimisation procedures \citep{wilson2018maximizing}.

In-context methods can reduce per-round costs by pre-training transformer architectures on functions drawn from priors from which sampling is efficient \citep{muller2022transformers}. Existing methods differ in which part of the BO process they amortize (Figure~\ref{fig:bo-amortisation}).
\begin{figure}[ht]
\centering
\expandafter\ifx\csname iflatexml\endcsname\relax
    \begin{minipage}[t]{0.65\textwidth}
        \vspace{0pt}
        \centering
        \includegraphics[trim=0cm 0cm 0cm 0cm, clip, width=\linewidth]{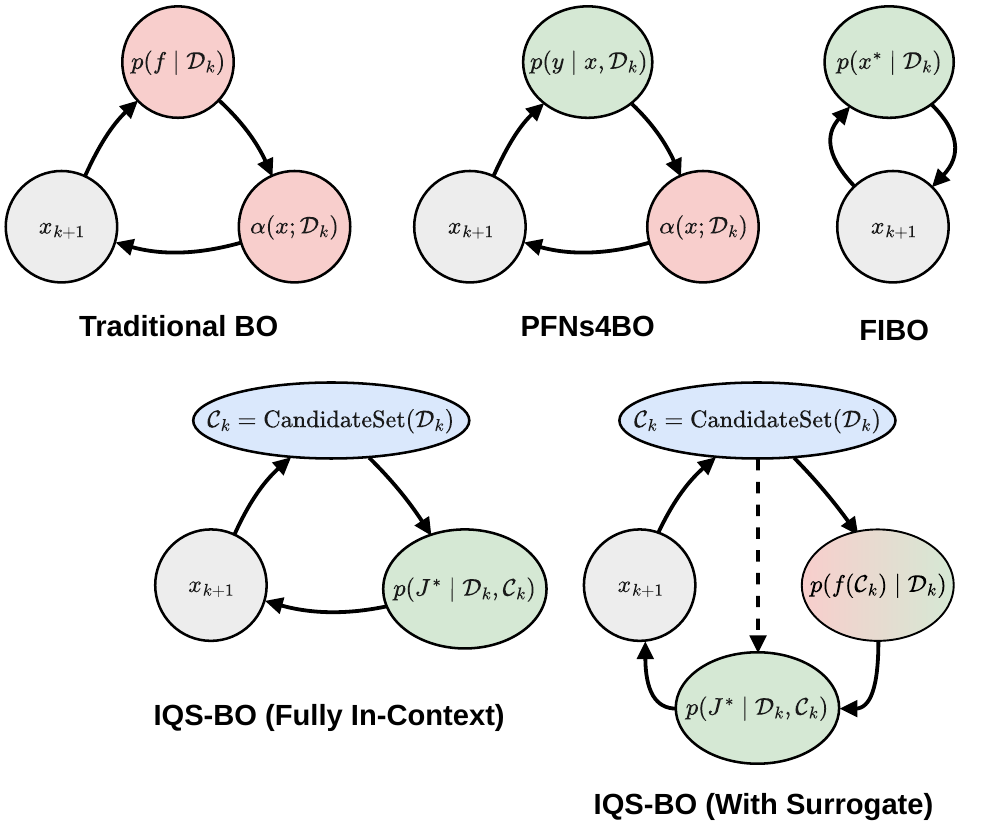}
    \end{minipage}\hfill
    \begin{minipage}[t]{0.34\textwidth}
        \vspace{0pt}
        \caption{One optimisation round of traditional BO, PFNs4BO \citep{muller2023pfns4bo}, FIBO \citep{fibo2025} and IQS-BO with and without a surrogate. $\mathcal{D}_k$: observations after $k$ evaluations, $x_{k+1}$: next query, $\mathcal{C}_k$: candidates drawn from $\mathcal{D}_k$, $\alpha$: acquisition function, $J^*$: index of the best candidate. Red: fitted or optimised every round. Green: one forward pass of a pre-trained network. Red-to-green: fixed surrogate of either kind. Blue: sampling without learning. Grey: quantities passed between steps.}
        \label{fig:bo-amortisation}
    \end{minipage}
    \vspace{-20pt}
\else
    \includegraphics[width=\linewidth]{figs/bo_amort_comparison.drawio.pdf}
    \caption{One optimisation round of traditional BO, PFNs4BO \citep{muller2023pfns4bo}, FIBO \citep{fibo2025} and IQS-BO with and without a surrogate. $\mathcal{D}_k$: observations after $k$ evaluations, $x_{k+1}$: next query, $\mathcal{C}_k$: candidates drawn from $\mathcal{D}_k$, $\alpha$: acquisition function, $J^*$: index of the best candidate. Red: fitted or optimised every round. Green: one forward pass of a pre-trained network. Red-to-green: fixed surrogate of either kind. Blue: sampling without learning. Grey: quantities passed between steps.}
    \label{fig:bo-amortisation}
\fi
\end{figure}

\pfn\, \citep{muller2023pfns4bo} amortizes the surrogate with a Prior-data Fitted Network (PFN) which returns the posterior predictive of each input in a single forward pass. Surrogate fitting is no longer required, since the model adapts to new observations in context, but an acquisition function must still be specified and numerically maximised at every round. Moreover, the acquisition function requires a closed form on the PFN's output.

\fibo\, \citep{fibo2025} removes the surrogate fit as well as the choice and maximisation of the acquisition function.
A pre-trained PFN encoder embeds the observations, and a normalizing flow conditioned on that embedding models the
distribution of the optimum's location which is sampled directly to obtain proposals. The learned object is a
density over the search space, and the authors show that sampling from it amounts to Thompson sampling under the pre-training prior. The method is especially suitable for high-throughput optimisation, since a batch of any size can be drawn almost for free as a set of i.i.d. draws, but implies a fixed query selection policy. %

Relevant to our work, learned acquisition functions replace the analytic rule by a network that scores a finite candidate set. NAP \citep{maraval2023nap} uses a transformer neural process that reads the observations and candidates and returns a score for each, learning a surrogate and the acquisition rule jointly. POLAR \citep{huang2026polar} learns a similar score on top of the representations of a pretrained predictive model. Since the best point to query is unknown on real data, both learn the score from rewards, by reinforcement learning or policy gradient on rollouts over solved tasks.

To the best of our knowledge, no method learns the query decision by supervised learning on synthetic priors. Under such priors the training label is free, since each pre-training function can be evaluated at every candidate, so the best one is known and query selection becomes a classification problem. This is the basis for In-context Query Selection for Bayesian Optimisation (\ours).

Concretely, \ours{} trains a transformer that reads the observations and a candidate set and predicts, for each candidate, the probability that it has the largest function value in the set. The network is trained with a cross-entropy loss, whose minimiser is the posterior probability that each candidate is the best. At deployment, \ours{} queries the candidate with the highest predicted probability, which approximates the Bayes-optimal choice among the candidates under the pre-training prior.

The proposed framework admits two implementations. First, the network can be trained directly on the observations, without a surrogate. This yields a model for fully in-context BO that, like FIBO, requires neither surrogate fitting nor acquisition function optimisation at deployment. Second, the network can take as additional input the predictions of a fixed probabilistic surrogate. We show this with a released \pfn{} model and with a Gaussian process whose hyperparameters are fitted to each synthetic task during pre-training and refitted at every iteration during inference.

Empirically, \ours{}, pre-trained for a fixed input dimension and batch size, (i)~outperforms \pfn{} pre-trained on the same prior and deployed via acquisition-function maximisation with L-BFGS-B, with gains increasing as both context size and batch size grow; (ii)~matches or outperforms GP-based BO, \fibo{} and NAP across real and synthetic benchmarks in three and four dimensions; and (iii)~excels in batch optimisation, reaching a better performance than \fibo{} at every batch size while maintaining comparable computational cost and proposing evaluations one to three orders of magnitude faster than \pfn{} and GP-based baselines.

Finally, we propose and evaluate a pre-training prior that augments GP samples with function families that are poorly captured by fixed-kernel GPs. We show that PFN-based models pre-trained on this prior achieve better optimisation performance than on priors that only include GP samples. The gains are largest when the initial context is small and become less pronounced as the context grows. 

\subsection{Summary of Main Contributions}
Our main contributions are:
\begin{itemize}
\item We formulate query selection in BO as supervised classification over a candidate set, where the class to predict is the index of the best candidate. Because this label is exact and cheap to compute for functions sampled from a prior, the selection rule can be learned with a standard cross-entropy loss.
\item We show that the cross-entropy objective is minimised by the posterior probability that a candidate is the best element of the candidate set.  Consequently, the learned policy approximates the Bayes-optimal decision rule induced by the pre-training prior. 
\item We introduce \ours{}, a transformer that selects the next query in-context by scoring all candidates in a single forward pass. The model can be pre-trained without a surrogate, enabling fully in-context BO, or combined with a fixed probabilistic surrogate, such as a pre-trained PFN or a GP, to amortise the query-selection step only. Batches are constructed sequentially, with each selection conditioned on the preceding ones.
\item We study the impact of the pre-training prior and propose a new prior that combines GP samples with warped, localised and piecewise-flat functions. We show that pre-training on this prior can improve the optimisation performance of PFN-based methods for BO. 
\end{itemize}

\section{Problem Statement}
\label{sec:problem}

We consider the maximisation of an expensive black-box function $f:\mathcal X\to\mathbb R$ over $\mathcal X=[0,1]^d$. At each optimisation round, we have observed $\mathcal D=\{(x_i,y_i)\}_{i=1}^{n}$, with noisy evaluations $y_i=f(x_i)+\varepsilon_i$, and must choose the next point at which to evaluate $f$. Standard BO makes this choice by fitting a probabilistic surrogate $p(f\mid\mathcal D)$ and maximising an acquisition function $\alpha$ over $\mathcal X$.

Our goal is to replace the maximisation of the acquisition function or, more generally, the selection of the next query, by a decision rule that is learned in-context. At each optimisation round, a stochastic sampler constructs a finite candidate set $\mathcal C=\{c_1,\ldots,c_M\}\subset\mathcal X$ of $M$ points from $\mathcal D$ (Appendix~\ref{app:candset}). Instead of maximising $\alpha$ over $\mathcal X$, we learn to select one candidate from $\mathcal C$ directly. Let
\begin{equation}
J^* = \argmax_{k=1,\ldots,M} f(c_k)
\label{eq:best-candidate}
\end{equation}
denote the index of the best candidate. Given the context $\mathcal D$ and the candidate set $\mathcal C$, we select the candidate with the largest posterior probability of achieving the highest function value in the set, i.e., 
\begin{equation} 
j^\star = \argmax_{j=1,\ldots,M} \mathbb P(J^*=j \mid \mathcal D,\mathcal C).
\label{eq:targetmax} 
\end{equation}

\section{Methodology}
\label{sec:method}
We now describe how the network $p_\theta$, with parameters $\theta$, is trained to approximate the posterior
\begin{equation}
\mathbb P(J^*=j \mid \mathcal D,\mathcal C),\qquad j=1,\ldots,M,
\label{eq:target}
\end{equation}
and how it selects the next query at inference. The network reads the observations and the candidates and returns (\ref{eq:target}) for every candidate in a single forward pass. It is pre-trained on functions sampled from a synthetic prior, so that it selects queries in-context without any parameter update, and it can take the predictions of a fixed surrogate $\phi$ as additional input (Section~\ref{sec:pretraining}). We first state the training objective (Section~\ref{sec:objective}) and its Bayesian interpretation (Section~\ref{sec:interpretation}), then describe the architecture (Section~\ref{sec:architecture}), pre-training (Section~\ref{sec:pretraining}) and inference (Section~\ref{sec:inference}). We denote our method by \ours.

\subsection{Training Objective}
\label{sec:objective}
The network $p_\theta$ is pre-trained on synthetic tasks $(\mathcal D,\mathcal C,J^*)$, obtained by drawing a function $f$ from a pre-training prior $p(f)$, a context $\mathcal D$ of noisy evaluations of $f$, a candidate set $\mathcal C$ from $\mathcal D$, and the label $J^*$ of (\ref{eq:best-candidate}). Since $f$ is known during pre-training, it can be evaluated without noise at every candidate, so $J^*$ is computed exactly, at the cost of $M$ function evaluations and one $\argmax$. The parameters $\theta$ are learned by minimising the expected cross-entropy
\begin{equation}
\mathcal L(\theta)=\mathbb E\bigl[-\log p_\theta(J^*\mid\mathcal D,\mathcal C)\bigr],
\label{eq:cross-entropy}
\end{equation}
where the expectation is over tasks generated as above. In practice, (\ref{eq:cross-entropy}) is estimated by its average over mini-batches of sampled tasks. Its minimiser is the posterior (\ref{eq:target}), so pre-training reduces to a multiclass classification problem whose optimal solution is the posterior probability that each candidate is the best in the set, as shown in Proposition~\ref{prop:posterior-best}.

\subsection{Posterior-Best-Candidate Interpretation}
\label{sec:interpretation}

The cross-entropy objective (\ref{eq:cross-entropy}) admits a simple Bayesian interpretation. Write $J^*=J^*(f,\mathcal C)$ for the label (\ref{eq:best-candidate}), viewed as a function of the sampled function $f$ and the candidate set $\mathcal C$. Training tasks are generated as $f\sim p(f)$, $\mathcal D\sim p(\mathcal D\mid f)$ and $\mathcal C\sim p(\mathcal C\mid\mathcal D)$, so (\ref{eq:cross-entropy}) is the population loss under the joint law of $(f,\mathcal D,\mathcal C)$. Since the candidate sampler uses only $\mathcal D$ and independent randomness, $\mathcal C\perp f\mid\mathcal D$. The network output $p_\theta(\cdot\mid\mathcal D,\mathcal C)$ is a probability mass function on $\{1,\ldots,M\}$.

\begin{proposition}\label{prop:posterior-best}
If the model is well specified, that is, $p_\theta(\cdot\mid\mathcal D,\mathcal C)=\mathbb P(J^*=\cdot\mid\mathcal D,\mathcal C)$ almost surely for some $\theta$, then every minimiser $\theta^*$ of $\mathcal L$ satisfies, almost surely and for all $j$,
\[
p_{\theta^*}(j\mid\mathcal D,\mathcal C)
=\mathbb P(J^*=j\mid\mathcal D,\mathcal C)
=\mathbb P_{f\sim p(f\mid\mathcal D)}\big(J^*(f,\mathcal C)=j\big).
\]
\end{proposition}

\begin{proof}
Let $\pi_j=\mathbb P(J^*=j\mid\mathcal D,\mathcal C)$ and $H(\pi)=-\sum_j\pi_j\log\pi_j\le\log M$. Conditioning on $(\mathcal D,\mathcal C)$, with all outer expectations below taken over $(\mathcal D,\mathcal C)$,
\[
\mathcal L(\theta)
=\mathbb E\Big[-\sum_{j=1}^M\pi_j\log p_\theta(j\mid\mathcal D,\mathcal C)\Big]
=\mathbb E\big[H(\pi)\big]+\mathbb E\Big[\mathrm{KL}\big(\pi\,\big\|\,p_\theta(\cdot\mid\mathcal D,\mathcal C)\big)\Big].
\]
The first term does not depend on $\theta$, and the second is non-negative and vanishes if and only if $p_\theta(\cdot\mid\mathcal D,\mathcal C)=\pi$ almost surely. Well-specification makes the value $\mathbb E[H(\pi)]$ attainable, so the minimisers are exactly such $\theta$. Finally, $J^*$ is a function of $(f,\mathcal C)$, and $p(f\mid\mathcal D,\mathcal C)=p(f\mid\mathcal D)$ because $\mathcal C\perp f\mid\mathcal D$, which gives the last equality.
\end{proof}

Consequently, a network trained by minimising (\ref{eq:cross-entropy}) approximates the target (\ref{eq:target}), the posterior probability that each candidate is the best in the set. Under a 0--1 loss on the identity of the best candidate, selecting $\argmax_j\mathbb P(J^*=j\mid\mathcal D,\mathcal C)$ is the Bayes-optimal decision. The network, which selects $\argmax_j p_\theta(j\mid\mathcal D,\mathcal C)$, therefore approximates the Bayes-optimal myopic decision rule on the candidate set.

\subsection{Architecture}
\label{sec:architecture}
The network $p_\theta$ is a transformer that reads $\mathcal D$ and $\mathcal C$ and returns a probability for every candidate in a single forward pass. Each context point $(x_i,y_i)$ is represented by the token $s=i$ and each candidate $c_j$ by the token $s=n+j$, giving $n+M$ tokens, each with a feature vector $t_s$ that includes the coordinates of the point and its relation to the data, and differs between context and candidate tokens (Appendix~\ref{app:model}). If a fixed surrogate $\phi$ is used, its outputs at each token, collected in a vector $m_s$, are added to the embedding
\begin{equation}
h^{(0)}_s=W_{\mathrm{feat}}\,t_s+W_{\phi}\,m_s,
\label{eq:embedding}
\end{equation}
where $W_{\mathrm{feat}}$ and $W_{\phi}$ are learned and the second term is absent without a surrogate.

The tokens are processed by a transformer encoder with $L$ layers and no positional encoding. The encoder maps the embeddings $h^{(0)}_s$ to final states $h^{(L)}_s$, and a scoring head $g_\theta$ maps the final state of each candidate token to a score $z_j=g_\theta\bigl(h^{(L)}_{n+j}\bigr)$. A softmax over the candidates gives
\begin{equation}
p_\theta(j\mid\mathcal D,\mathcal C)=\frac{\exp(z_j)}{\sum_{k=1}^{M}\exp(z_k)},\qquad j=1,\ldots,M.
\label{eq:softmax}
\end{equation}
Context tokens attend only to each other and candidate tokens only to context tokens, so each score depends only on $\mathcal D$ and its own candidate, and $M$ can differ between pre-training and inference (Appendix~\ref{app:architecture}).

\subsection{Pre-Training}
\label{sec:pretraining}
Pre-training only requires sampling functions from a prior and evaluating them at arbitrary points. We use two priors. The first is the HEBO prior of PFNs4BO \citep{muller2023pfns4bo}, sampled with the hyperparameters of its released model. The second is a mixture prior that combines GP samples with three types of functions that stationary GP surrogates model poorly, namely input-warped GP samples, sums of local Gaussian bumps and piecewise-flat plateau functions (Appendix \ref{app:prior}).

Without a surrogate, the network reads only the observed data and is trained as in Algorithm~\ref{alg:pretrain}. It learns from $\mathcal D$ both what the data imply about $f$ and which candidate to select, and the result is a fully in-context model that requires neither surrogate fitting nor an acquisition function at inference.

With a surrogate $\phi$, Algorithm~\ref{alg:pretrain} has one additional step. Before the forward pass, $\phi$ is evaluated on the context of the task, and its outputs $m_s$ enter the embedding (\ref{eq:embedding}). The loss, the data and the rest of the network are unchanged. In our experiments we use two surrogates:
\begin{itemize}
\item \ourspfn{} uses the released \pfnhebo{} model \citep{muller2023pfns4bo}, with parameters kept frozen during training. The network reads a learned combination of the PFN's hidden states at every token and five summaries of its predictive distribution at every candidate. Since both the surrogate and \ours\, are in-context, this model performs BO fully in-context.
\item \oursgp{} uses a GP with a Mat\'ern-5/2 kernel with automatic relevance determination (ARD), whose hyperparameters are fitted by maximum a posteriori (MAP) estimation to the context of each pre-training task. The network reads the posterior mean, the posterior standard deviation, the posterior correlation with the best context point and the fitted hyperparameters. At inference, the GP is refitted at every evaluation, as in standard BO.
\end{itemize}
All three variants can be pre-trained to select batches of $q$ queries (Appendix~\ref{app:batch}). The $q$ points are selected one at a time, and each earlier pick is passed to the network as a pending token, so that it learns to take previous decisions into account. Selecting the best $q$ points from a candidate set is a different decision problem for every $q$. For optimal performance, each model was pre-trained for a specific batch size. Depending on the variant, pre-training takes between 2.7 and 6.1 hours on a single NVIDIA H100 GPU (Appendix~\ref{app:training}).

\begin{algorithm}[t]
\caption{Pre-training \ours{} (no surrogate, single query)}
\label{alg:pretrain}
\begin{algorithmic}[1]
\Require prior $p(f)$, network parameters $\theta$, candidate-set size $M$
\For{each optimiser step}
  \State Sample $f\sim p(f)$ and a context $\mathcal D=\{(x_i,y_i)\}_{i=1}^{n}$ from $f$ \Comment{Section~\ref{sec:pretraining}}
  \State Draw $\mathcal C=\textsc{CandidateSet}(\mathcal D,M)$ \Comment{Appendix~\ref{app:candset}}
  \State Compute $J^*$ by (\ref{eq:best-candidate}) on noise-free $f$
  \State Compute $p_\theta(\cdot\mid\mathcal D,\mathcal C)$ and take a gradient step on $-\log p_\theta(J^*\mid\mathcal D,\mathcal C)$ \Comment{Section~\ref{sec:architecture}}
\EndFor
\end{algorithmic}
\end{algorithm}

\begin{algorithm}[t]
\caption{One optimisation round of \ours{} (single query)}
\label{alg:bo_round}
\begin{algorithmic}[1]
\Require context $\mathcal D$, network $\theta$, surrogate $\phi$ or $\varnothing$, candidate-set size $M$
\State Draw $\mathcal C=\textsc{CandidateSet}(\mathcal D,M)$
\State If $\phi\neq\varnothing$, compute the surrogate outputs $m=\phi(\mathcal D,\mathcal C)$
\State Compute $p_\theta(\cdot\mid\mathcal D,\mathcal C)$ in one forward pass, with $m$ as input if available
\State Select $x_{\mathrm{next}}=c_{\hat\jmath}$ with $\hat\jmath=\argmax_{j\le M}p_\theta(j\mid\mathcal D,\mathcal C)$
\State Evaluate $f$ at $x_{\mathrm{next}}$, observe $y_{\mathrm{next}}$ and set $\mathcal D\gets\mathcal D\cup\{(x_{\mathrm{next}},y_{\mathrm{next}})\}$
\end{algorithmic}
\end{algorithm}

\subsection{Inference}
\label{sec:inference}
After pre-training, the network is deployed without any parameter update (Algorithm~\ref{alg:bo_round}). At every round a new candidate set with $M$ elements is drawn, one forward pass gives $p_\theta(\cdot\mid\mathcal D,\mathcal C)$, and the candidate with the highest probability is evaluated. A round of \ours{} costs one forward pass, \ourspfn{} adds one forward pass of the frozen PFN, and \oursgp{} adds one GP fit. Batches are selected sequentially with pending tokens, as described in Section~\ref{sec:pretraining} and Appendix ~\ref{app:batch}.

\section{Empirical Evaluation}
\label{sec:results}

This section evaluates \ours{} in four steps. Section~\ref{sec:results-main} compares its three variants, without a surrogate, with a PFNs4BO surrogate (\ourspfn{}) and with a GP surrogate (\oursgp{}), against GP-based and in-context baselines. Section~\ref{sec:results-controlled} isolates the effects of the learned query selection and of the mixture prior. Section~\ref{sec:results-batch} evaluates batch optimisation, and Section~\ref{sec:ablation} ablates the design of the candidate set and the sequential batch rule. Implementation details, datasets, baselines, metrics and additional results are given in Appendices~\ref{app:model}--\ref{app:results}.

\subsection{Experimental Setup}
\label{sec:setup}
\label{sec:benchmarks}
We use ten 3D problems, eight synthetic and two real-world, and five 4D problems, four synthetic and one real-world, two of which were chosen to be suited for GP surrogates (Appendix~\ref{app:benchmarks}). As baselines, we use GP-based BO with MAP (\gpmap{}) and full Bayes (\gpfb{}), as well as two released PFNs4BO models, \fibo{}, NAP \citep{maraval2023nap} and Sobol search (Appendix~\ref{app:baselines}). 

All runs start from the same Sobol design of $n_{\text{init}}\in\{10,20,40\}$ points and run for $T=30$ rounds with $q=2$, with 90 seeds per problem. We report the regret $\tilde r(t)=f^*-f(x_t^{+})$ after $t$ rounds, where $x_t^{+}$ is the incumbent, clipped at zero and normalised by the mean regret of the initial design on each problem. Each run is summarised by the average regret $R^{\text{AUC}}=\frac{1}{T}\sum_{t=1}^{T}\tilde r(t)$, which rewards reaching low regret early, and by $R^{\text{final}}=\tilde r(T)$, which measures the regret at the final round. To compare methods across problems, we rank them on each problem and seed by these quantities, with $1$ being best, and report the average ranks. For pairwise comparisons, we report the win rate, the share of paired runs, with the same problem, seed and initial design, in which a method attains the lower regret. For batch optimisation, we report the GAP \citep{huang2006global}, $(y_t-y_0)/(y^*-y_0)$, where $y_t$ is the best value after $t$ evaluations, $y_0$ the best value in the initial design and $y^*$ the optimum.

\subsection{Main Results}
\label{sec:results-main}
Figure~\ref{fig:main} shows the average regret and the average rank per round with a starting context of $n_{\text{init}}=40$ on the 3D problems, and Table~\ref{tab:rank} reports the average ranks at every context on the 3D and 4D problems. In both, methods pre-trained on the mixture prior are compared separately from methods pre-trained on the HEBO or GP-based priors. \gpmap{}, \gpfb{} and Sobol require no pre-training and are included in all comparisons. Results at $n_{\text{init}}\in\{10,20\}$ and per-problem results are given in Appendix~\ref{app:full3d}.

\begin{figure}[t!]
\centering
\setlength{\abovecaptionskip}{2pt}
\includegraphics[width=\textwidth]{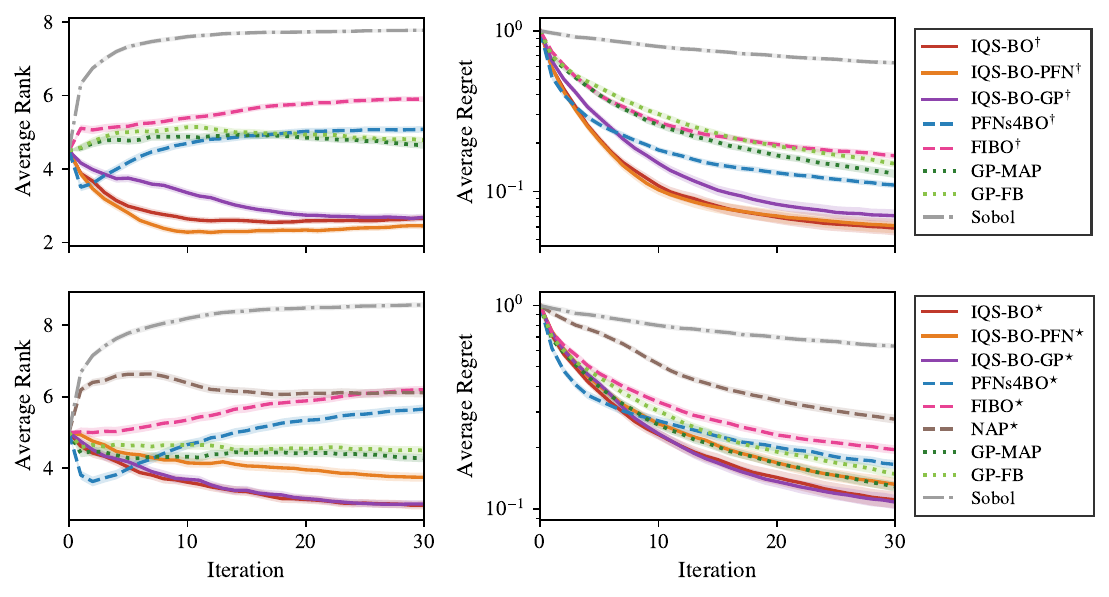}
\caption{Average rank (left) and average regret (right) per iteration, ten 3D problems, $n_{\text{init}}=40$, 90 seeds, $q=2$, with $95\%$ confidence intervals. Top, methods pre-trained on the mixture prior ($^{\dagger}$). Bottom, methods pre-trained on the HEBO prior or on GP draws ($^{\star}$). Ranks are computed within each row, and regret is normalised by the mean initial regret of each problem.}
\label{fig:main}
\vspace{-10pt}
\end{figure}

\begin{table}[H]
\centering\footnotesize
\setlength{\abovecaptionskip}{2pt}
\setlength{\belowcaptionskip}{5pt}
\renewcommand{\arraystretch}{0.85}
\caption{Mean rank on $R^{\text{AUC}}$ (left) and $R^{\text{final}}$ (right), over 90 seeds and the ten 3D or five 4D problems, $\pm$ $95\%$ bootstrap CI half-width. $^{\dagger}$ mixture prior, $^{\star}$ HEBO or GP prior. Ranks are computed within each block, and bold marks the best rank per column.}
\label{tab:rank}
\setlength{\tabcolsep}{1pt}
\begin{tabular*}{\linewidth}{@{}l@{\extracolsep{\fill}}cccccc@{\hspace{8pt}}cccccc@{}}
\toprule
 & \multicolumn{6}{c}{$R^{\text{AUC}}$} & \multicolumn{6}{c}{$R^{\text{final}}$} \\
\cmidrule(lr){2-7}\cmidrule(l){8-13}
 & \multicolumn{3}{c}{$d=3$} & \multicolumn{3}{c}{$d=4$} & \multicolumn{3}{c}{$d=3$} & \multicolumn{3}{c}{$d=4$} \\
\cmidrule(lr){2-4}\cmidrule(lr){5-7}\cmidrule(lr){8-10}\cmidrule(l){11-13}
$n_{\text{init}}$ & 10 & 20 & 40 & 10 & 20 & 40 & 10 & 20 & 40 & 10 & 20 & 40 \\
\midrule
\ours{}$^{\dagger}$ & \textbf{3.1}{\scriptsize$\pm$.1} & 2.8{\scriptsize$\pm$.1} & 2.8{\scriptsize$\pm$.1} & 3.3{\scriptsize$\pm$.1} & 3.5{\scriptsize$\pm$.1} & 3.4{\scriptsize$\pm$.1} & \textbf{2.8}{\scriptsize$\pm$.1} & 2.6{\scriptsize$\pm$.1} & 2.7{\scriptsize$\pm$.1} & 3.4{\scriptsize$\pm$.1} & 3.5{\scriptsize$\pm$.1} & 3.3{\scriptsize$\pm$.1} \\
\ourspfn{}$^{\dagger}$ & 3.1{\scriptsize$\pm$.1} & \textbf{2.7}{\scriptsize$\pm$.1} & \textbf{2.6}{\scriptsize$\pm$.1} & 3.6{\scriptsize$\pm$.1} & 3.5{\scriptsize$\pm$.1} & 3.3{\scriptsize$\pm$.1} & 2.8{\scriptsize$\pm$.1} & \textbf{2.5}{\scriptsize$\pm$.1} & \textbf{2.5}{\scriptsize$\pm$.1} & 3.5{\scriptsize$\pm$.1} & 3.5{\scriptsize$\pm$.1} & 3.3{\scriptsize$\pm$.1} \\
\oursgp{}$^{\dagger}$ & 3.6{\scriptsize$\pm$.1} & 3.3{\scriptsize$\pm$.1} & 3.3{\scriptsize$\pm$.1} & 3.6{\scriptsize$\pm$.1} & \textbf{3.2}{\scriptsize$\pm$.1} & \textbf{3.1}{\scriptsize$\pm$.1} & 2.8{\scriptsize$\pm$.1} & 2.8{\scriptsize$\pm$.1} & 2.7{\scriptsize$\pm$.1} & \textbf{3.2}{\scriptsize$\pm$.1} & \textbf{3.0}{\scriptsize$\pm$.1} & \textbf{2.9}{\scriptsize$\pm$.1} \\
PFNs4BO$^{\dagger}$ & 3.8{\scriptsize$\pm$.1} & 4.2{\scriptsize$\pm$.1} & 4.2{\scriptsize$\pm$.1} & 4.1{\scriptsize$\pm$.1} & 4.4{\scriptsize$\pm$.1} & 5.1{\scriptsize$\pm$.1} & 4.9{\scriptsize$\pm$.1} & 5.1{\scriptsize$\pm$.1} & 5.1{\scriptsize$\pm$.1} & 5.2{\scriptsize$\pm$.1} & 5.3{\scriptsize$\pm$.1} & 5.4{\scriptsize$\pm$.1} \\
\fibo{}$^{\dagger}$ & 5.5{\scriptsize$\pm$.1} & 5.6{\scriptsize$\pm$.1} & 5.6{\scriptsize$\pm$.1} & 6.1{\scriptsize$\pm$.1} & 5.8{\scriptsize$\pm$.1} & 5.3{\scriptsize$\pm$.1} & 5.8{\scriptsize$\pm$.1} & 5.9{\scriptsize$\pm$.1} & 5.9{\scriptsize$\pm$.1} & 5.6{\scriptsize$\pm$.1} & 5.6{\scriptsize$\pm$.1} & 5.6{\scriptsize$\pm$.1} \\
\gpmap{} & 4.3{\scriptsize$\pm$.1} & 4.5{\scriptsize$\pm$.1} & 4.8{\scriptsize$\pm$.1} & \textbf{3.2}{\scriptsize$\pm$.1} & 3.3{\scriptsize$\pm$.1} & 3.6{\scriptsize$\pm$.1} & 4.5{\scriptsize$\pm$.1} & 4.6{\scriptsize$\pm$.1} & 4.6{\scriptsize$\pm$.1} & 3.5{\scriptsize$\pm$.1} & 3.5{\scriptsize$\pm$.1} & 3.8{\scriptsize$\pm$.1} \\
\gpfb{} & 5.2{\scriptsize$\pm$.1} & 5.1{\scriptsize$\pm$.1} & 4.9{\scriptsize$\pm$.1} & 4.5{\scriptsize$\pm$.1} & 4.5{\scriptsize$\pm$.1} & 4.3{\scriptsize$\pm$.1} & 4.6{\scriptsize$\pm$.1} & 4.7{\scriptsize$\pm$.1} & 4.8{\scriptsize$\pm$.1} & 3.9{\scriptsize$\pm$.1} & 3.7{\scriptsize$\pm$.1} & 3.8{\scriptsize$\pm$.1} \\
Sobol & 7.5{\scriptsize$\pm$.1} & 7.7{\scriptsize$\pm$.1} & 7.8{\scriptsize$\pm$.1} & 7.6{\scriptsize$\pm$.1} & 7.9{\scriptsize$\pm$.0} & 7.9{\scriptsize$\pm$.0} & 7.7{\scriptsize$\pm$.1} & 7.7{\scriptsize$\pm$.1} & 7.8{\scriptsize$\pm$.0} & 7.8{\scriptsize$\pm$.0} & 7.9{\scriptsize$\pm$.0} & 7.9{\scriptsize$\pm$.0} \\
\midrule
\ours{}$^{\star}$ & 4.2{\scriptsize$\pm$.1} & \textbf{3.9}{\scriptsize$\pm$.1} & 3.6{\scriptsize$\pm$.1} & 4.7{\scriptsize$\pm$.2} & 4.4{\scriptsize$\pm$.2} & 3.9{\scriptsize$\pm$.1} & 3.1{\scriptsize$\pm$.1} & 3.0{\scriptsize$\pm$.1} & \textbf{3.0}{\scriptsize$\pm$.1} & 3.8{\scriptsize$\pm$.2} & 3.5{\scriptsize$\pm$.1} & 3.2{\scriptsize$\pm$.1} \\
\ourspfn{}$^{\star}$ & 5.0{\scriptsize$\pm$.1} & 4.6{\scriptsize$\pm$.1} & 4.3{\scriptsize$\pm$.1} & 5.1{\scriptsize$\pm$.2} & 4.9{\scriptsize$\pm$.2} & 4.6{\scriptsize$\pm$.2} & 4.4{\scriptsize$\pm$.1} & 4.1{\scriptsize$\pm$.1} & 3.7{\scriptsize$\pm$.1} & 4.3{\scriptsize$\pm$.2} & 4.2{\scriptsize$\pm$.2} & 4.1{\scriptsize$\pm$.1} \\
\oursgp{}$^{\star}$ & 4.4{\scriptsize$\pm$.1} & 4.0{\scriptsize$\pm$.1} & \textbf{3.6}{\scriptsize$\pm$.1} & 4.7{\scriptsize$\pm$.2} & 3.9{\scriptsize$\pm$.2} & 3.7{\scriptsize$\pm$.1} & \textbf{2.9}{\scriptsize$\pm$.1} & \textbf{2.8}{\scriptsize$\pm$.1} & 3.0{\scriptsize$\pm$.1} & 3.8{\scriptsize$\pm$.2} & \textbf{3.1}{\scriptsize$\pm$.1} & \textbf{2.7}{\scriptsize$\pm$.1} \\
PFNs4BO$^{\star}$ & \textbf{3.6}{\scriptsize$\pm$.1} & 3.9{\scriptsize$\pm$.1} & 4.5{\scriptsize$\pm$.1} & 3.9{\scriptsize$\pm$.2} & 4.7{\scriptsize$\pm$.2} & 5.1{\scriptsize$\pm$.2} & 5.1{\scriptsize$\pm$.1} & 5.4{\scriptsize$\pm$.1} & 5.7{\scriptsize$\pm$.1} & 5.1{\scriptsize$\pm$.2} & 5.7{\scriptsize$\pm$.1} & 6.0{\scriptsize$\pm$.1} \\
\fibo{}$^{\star}$ & 5.3{\scriptsize$\pm$.1} & 5.5{\scriptsize$\pm$.1} & 5.4{\scriptsize$\pm$.1} & 5.5{\scriptsize$\pm$.2} & 5.5{\scriptsize$\pm$.1} & 5.6{\scriptsize$\pm$.1} & 6.3{\scriptsize$\pm$.1} & 6.3{\scriptsize$\pm$.1} & 6.2{\scriptsize$\pm$.1} & 5.8{\scriptsize$\pm$.1} & 6.1{\scriptsize$\pm$.1} & 6.1{\scriptsize$\pm$.1} \\
NAP$^{\star}$ & 6.0{\scriptsize$\pm$.1} & 6.1{\scriptsize$\pm$.1} & 6.5{\scriptsize$\pm$.1} & 6.0{\scriptsize$\pm$.2} & 5.7{\scriptsize$\pm$.2} & 5.9{\scriptsize$\pm$.1} & 6.2{\scriptsize$\pm$.1} & 6.1{\scriptsize$\pm$.1} & 6.1{\scriptsize$\pm$.1} & 6.9{\scriptsize$\pm$.1} & 6.8{\scriptsize$\pm$.1} & 6.6{\scriptsize$\pm$.1} \\
\gpmap{} & 3.6{\scriptsize$\pm$.1} & 3.9{\scriptsize$\pm$.1} & 4.1{\scriptsize$\pm$.1} & \textbf{2.7}{\scriptsize$\pm$.1} & \textbf{2.8}{\scriptsize$\pm$.1} & \textbf{3.3}{\scriptsize$\pm$.1} & 4.2{\scriptsize$\pm$.1} & 4.4{\scriptsize$\pm$.1} & 4.3{\scriptsize$\pm$.1} & \textbf{3.2}{\scriptsize$\pm$.1} & 3.3{\scriptsize$\pm$.1} & 3.7{\scriptsize$\pm$.1} \\
\gpfb{} & 4.9{\scriptsize$\pm$.1} & 4.7{\scriptsize$\pm$.1} & 4.4{\scriptsize$\pm$.1} & 4.4{\scriptsize$\pm$.1} & 4.4{\scriptsize$\pm$.1} & 4.1{\scriptsize$\pm$.1} & 4.3{\scriptsize$\pm$.1} & 4.5{\scriptsize$\pm$.1} & 4.5{\scriptsize$\pm$.1} & 3.7{\scriptsize$\pm$.1} & 3.6{\scriptsize$\pm$.1} & 3.7{\scriptsize$\pm$.1} \\
Sobol & 8.0{\scriptsize$\pm$.1} & 8.4{\scriptsize$\pm$.1} & 8.6{\scriptsize$\pm$.1} & 8.2{\scriptsize$\pm$.1} & 8.7{\scriptsize$\pm$.1} & 8.8{\scriptsize$\pm$.1} & 8.3{\scriptsize$\pm$.1} & 8.5{\scriptsize$\pm$.1} & 8.6{\scriptsize$\pm$.1} & 8.5{\scriptsize$\pm$.1} & 8.7{\scriptsize$\pm$.0} & 8.8{\scriptsize$\pm$.0} \\
\bottomrule
\end{tabular*}

\vspace{-10pt}
\end{table}

\paragraph{Mixture Prior}
Under the mixture prior, all three variants of \ours{} consistently outperform the competing learned methods. On the 3D benchmarks, they achieve the best mean rank on both $R^{\text{AUC}}$ and $R^{\text{final}}$ across all initial context sizes, substantially outperforming PFNs4BO and \fibo{} despite being trained on the same prior. Figure~\ref{fig:main} shows that the three variants behave similarly, indicating that the gain is mostly due to the learned query selection rather than to the surrogate choice. In 4D the GP baselines become more competitive, partly because two of the five benchmark problems are well specified by GP surrogates. Nevertheless, \oursgp{} attains the best $R^{\text{final}}$ rank at every context size, while all three variants of \ours{} significantly outperform PFNs4BO and \fibo{}. Overall, the results indicate that \ours{} outperforms in-context methods pre-trained on the same prior, and matches or outperforms GP-based BO.

\paragraph{HEBO and GP-based Priors}
Similar results are found when comparing models trained on the HEBO or GP-based priors. At small context sizes, PFNs4BO and GP-based BO remain competitive in terms of average regret. As more observations become available, \ours{} achieves the best or near-best final regret across all 3D settings and the strongest final-regret performance in 4D for moderate and large context sizes.  PFNs4BO often performs well at the beginning of the optimization process, whereas \ours{} consistently reaches better solutions by the end of the budget. Finally, \ours{} achieves these results at substantially lower computational cost. Because each query requires only forward passes, without surrogate fitting or acquisition function maximisation, \ours{} proposes queries about $11$ times faster than \gpmap{} and more than two orders of magnitude faster than \pfnhebo{} (Appendix~\ref{app:bobatch}).

\subsection{Controlled Comparisons and Effect of the Prior}
\label{sec:results-controlled}

\begin{figure}[t]
\centering
\setlength{\abovecaptionskip}{2pt}
\includegraphics[width=\textwidth]{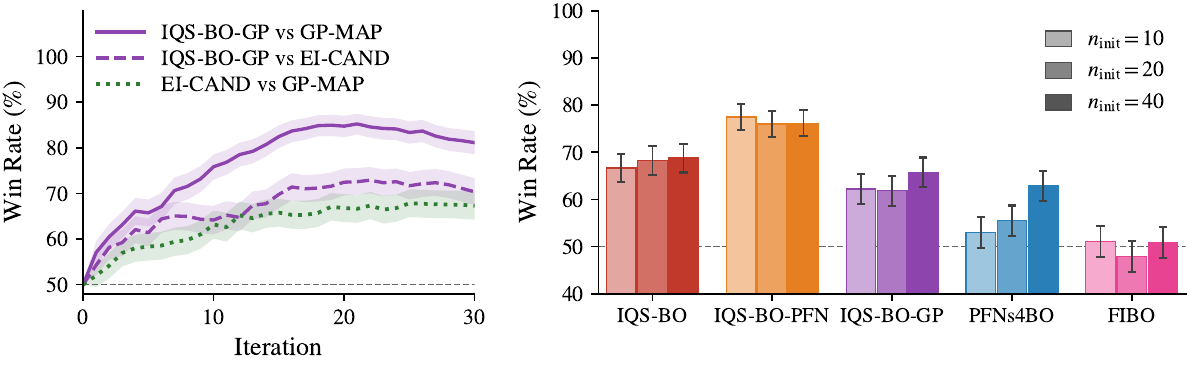}
\caption{Paired win rates (Section~\ref{sec:setup}), ten 3D problems, 900 paired runs, mean with $95\%$ CI. Left, on $\tilde r(t)$ per iteration, $n_{\text{init}}=40$. Right, on $R^{\text{final}}$, each model pre-trained on the mixture prior against its original prior.}
\label{fig:controlled}
\vspace{-15pt}
\end{figure}

\paragraph{Learned Query Selection Against Expected Improvement}
The methods compared in Section~\ref{sec:results-main} differ in more than how they select queries. To isolate the contribution of the network, we compare \oursgp{}, pre-trained on the mixture prior, with \gpmap{}, which fits the same GP at every round. \gpmap{} maximises expected improvement over the whole domain, whereas \oursgp{} lets the network choose among a finite set of candidates. We therefore add a control, \ctrl{}, which uses the GP, the candidate set and the batch rule of \oursgp{} but scores the candidates by expected improvement, so that \oursgp{} and \ctrl{} differ only in how candidates are scored (Appendix~\ref{app:controlled}). \oursgp{} outperforms \ctrl{} at every context size, and its advantage grows as more observations become available (Figure~\ref{fig:controlled}, left, and Table~\ref{tab:controlled}). \ctrl{} in turn outperforms \gpmap{}, mainly because its candidate set includes perturbations of the best observations (Appendix~\ref{app:ablations}). Taken together, these comparisons show that the network selects better queries than expected improvement from the same GP and the same candidates, and that both the network and the candidate design contribute to the gain of \oursgp{} over \gpmap{}.

\paragraph{Effect of the Prior}
We pre-train and evaluate other in-context BO methods on the mixture prior. Figure~\ref{fig:controlled} compares each model pre-trained on the mixture prior with the same model pre-trained on its original GP-based prior, which measures the gain of the mixture over pre-training on GP samples alone(Appendix~\ref{app:prior-results}). The mixture prior improves \ours{}, \ourspfn{} and PFNs4BO in most runs, most strongly in the early rounds, and the same pattern holds in 4D. This gain can be explained by the more complex and varied functional forms, such as local bumps and plateaus, which the models see during pre-training and which a GP prior alone specifies poorly. Nonetheless, models pre-trained on the mixture prior perform as well as their GP-prior counterparts on problems for which GP surrogates are well specified, such as Rosenbrock and the GP-prior draws, with the exception of a few cases discussed in Appendix~\ref{app:prior-results}. PFNs4BO gains less than \ours{}, likely because its prior only shapes the posterior, which expected improvement reduces to a mean and a variance at each point, whereas in \ours{} the prior shapes the choice of query directly. \fibo{} is the exception, with a negligible effect, possibly because it compresses the observations into a single averaged context vector, which discards much of the added functional variety. In summary, the benefit of the mixture prior is not specific to \ours{}, as it also improves PFNs4BO, and it is largest at early evaluations.

\subsection{Batch Bayesian Optimisation}
\label{sec:results-batch}
In many applications, such as laboratory experiments, several queries are evaluated in parallel, and the batch size can be large \citep{sin2025highly}. We evaluate batch sizes $q\in\{10,20,50\}$ on seven of the 3D problems, starting from $n_{\text{init}}=40$ points and running to 200 evaluations. For each batch size, \ours{} and \ourspfn{} are pre-trained on the HEBO prior. Figure~\ref{fig:batch} reports the GAP after 200 evaluations against the proposal time, the total time spent selecting all queries of a run. Further results can be found in Appendix~\ref{app:bobatch}.

\begin{figure}[t!]
\centering
\setlength{\abovecaptionskip}{2pt}
\includegraphics[width=\textwidth]{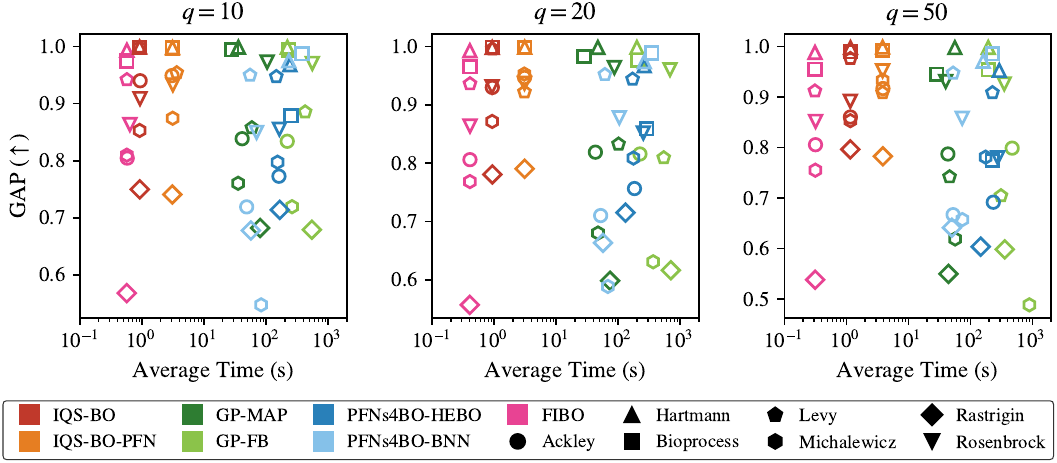}
\caption{GAP after 200 evaluations against the time spent proposing queries over a complete run, one marker per problem and method, seven 3D problems, $n_{\text{init}}=40$, mean over 90 seeds. Upper left is better.}
\label{fig:batch}
\vspace{-15pt}
\end{figure}

\paragraph{Performance}
\ours{} and \ourspfn{} reach the highest GAP at every batch size, and their margin over the baselines grows with the batch size. Every baseline loses quality as the batch grows, whereas \ours{} is almost unaffected, so that at $q=50$ it is still ahead of the best baseline at $q=10$. This robustness comes from pre-training at the deployed batch size, where the network learns to account for earlier picks in the batch and to diversify it (Appendix~\ref{app:ablations}). Our models have the best GAP on every problem except Rosenbrock, where \gpmap{} leads as in the sequential setting, and Hartmann at $q=50$, where \gpfb{} leads by a negligible margin. The margins are largest on the multimodal problems, Ackley, Levy and Rastrigin.

\paragraph{Computational Cost}
\ours{} achieves excellent batch-BO performance at a fraction of the cost of most baselines. Since each point of a batch requires one forward pass, proposing all queries of a run takes $0.9$ to $1.2$\,s, and about $3$\,s for \ourspfn{}, which also runs the frozen PFN. Only \fibo{} is faster, with $0.3$ to $0.6$\,s, because it samples the points of a batch independently, but it consistently reaches a lower GAP. The GP baselines and \pfnhebo{}, which maximise an acquisition function for every batch, are one to three orders of magnitude slower, and \ours{} is $38$ to $68$ times faster than \gpmap{}. Overall, \ours{} scales well to large batches, reaching the highest GAP at every batch size while proposing queries more than an order of magnitude faster than GP-based BO.

\subsection{Ablation Study}
\label{sec:ablation}
\begin{wraptable}{r}{0.47\textwidth}
\vspace{-22pt}
\centering\footnotesize
\setlength{\tabcolsep}{3pt}
\renewcommand{\arraystretch}{1.0}
\caption{Win rate (\%) on $R^{\text{final}}$ of \ours{} against each variant, means $\pm$$95\%$ CI over ten 3D problems and 900 paired runs.}
\label{tab:ablation}
\vspace{2pt}
\begin{tabular}{@{}lccc@{}}
\toprule
Variant & $n_{\text{init}}{=}10$ & $20$ & $40$ \\
\midrule
space-filling & $84.2{\scriptstyle\pm2.2}$ & $87.9{\scriptstyle\pm2.0}$ & $88.4{\scriptstyle\pm1.8}$ \\
$M{=}128$     & $57.9{\scriptstyle\pm3.1}$ & $56.7{\scriptstyle\pm3.2}$ & $56.6{\scriptstyle\pm3.1}$ \\
$M{=}2048$    & $48.3{\scriptstyle\pm3.2}$ & $46.4{\scriptstyle\pm3.2}$ & $50.3{\scriptstyle\pm3.2}$ \\
top-$q$ batch & $59.2{\scriptstyle\pm3.2}$ & $57.9{\scriptstyle\pm3.2}$ & $55.6{\scriptstyle\pm3.2}$ \\
\bottomrule
\end{tabular}
\vspace{-6pt}
\end{wraptable}

To conclude, we ablate three design choices of \ours{}, namely the composition and the size of the candidate set and the sequential selection of batch points (Table~\ref{tab:ablation}). The composition of the candidate set matters most. The network can only refine the incumbent if candidates lie near it, so replacing the perturbations of the best observations by space-filling candidates degrades performance in most runs. These perturbations also explain why \ctrl{} outperforms \gpmap{} in Section~\ref{sec:results-controlled}. Regarding the size of the candidate set, fewer candidates hurt, whereas more candidates yield only a modest improvement in $R^{\text{final}}$ over our default, and none in $R^{\text{AUC}}$. Finally, selecting batch points sequentially, with earlier picks as pending points, outperforms taking the best candidates in a single forward pass, so the network relies on the pending points to improve quality and diversify the batch.

\section{Conclusion}
\label{sec:conclusion}

We introduced \ours{}, a Bayesian optimisation method that selects the next query in-context. Instead of fitting a surrogate and maximising an acquisition function, \ours{} predicts the posterior probability that each candidate is the best in a set and selects the candidate with the highest probability. It outperforms in-context methods pre-trained on the same prior, matches or outperforms GP-based BO, and in batch BO reaches the highest performance while proposing queries more than an order of magnitude faster than GP-based BO. Finally, we showed that enriching the pre-training prior with functions that stationary GPs model poorly can further improve performance, especially at early evaluations.

Several directions for future work remain. The models are pre-trained for a fixed input dimension and batch size, and our experiments focus on three and four dimensions. Queries are limited to a finite candidate set whose design affects performance, and the objective is myopic, as it optimizes a one-step decision rule. Natural extensions are to train a single model applicable across dimensions and batch sizes, to learn the candidate set jointly with the selection, and to extend the objective beyond one-step selection.

\subsection*{AI use statement}
No generative AI tools were used as part of the core methodology of this work. The method, the experimental design, the implementation of all models and baselines, the data and the interpretation of results are entirely the work of the authors. Generative AI tools were used as writing assistance, to improve the readability of the text and to format references. This use does not affect the scientific content, methodology or originality of the research. The authors take full responsibility for the final content of this work.

\subsection*{Reproducibility statement}
All information needed to reproduce our results is given in the paper. Appendix~\ref{app:model} gives the full architecture and batch-selection rule, Appendix~\ref{app:datagen} the pre-training priors, candidate sampler and all training hyperparameters, and Appendix~\ref{app:baselines} the configuration of every baseline, including our reimplementation of FIBO. The benchmark problems and their global optima are defined in Appendix~\ref{app:benchmarks}, the evaluation protocol in Section~\ref{sec:setup}, and the complete results behind every figure and table in Appendix~\ref{app:results}.

\bibliography{references}
\bibliographystyle{preprint}

\clearpage
\appendix

\section{Comparison to In-Context and Learned Acquisition Methods}
\label{app:related}

This appendix expands the discussion of in-context BO and learned acquisition functions in Section~\ref{sec:intro}. We call a method in-context if, at deployment, the network reads the observations directly and nothing is fitted or updated. Appendix~\ref{app:incontext} reviews the in-context methods, Appendix~\ref{app:learned-af} the learned acquisition methods, which read a refitted GP posterior and are therefore not in-context, and Appendix~\ref{app:relation} relates both to \ours{}. Table~\ref{tab:pretrain-compare} summarises what each method learns, from which signal, and how it proposes the next query.

\begin{table}[h]
\centering\small\setlength{\tabcolsep}{4pt}
\caption{What each method learns, from which signal, and how it proposes the next query.}
\label{tab:pretrain-compare}
\begin{tabular}{@{}lllll@{}}
\toprule
Method & Learned object & Training signal & Surrogate & Next query \\
\midrule
PFNs4BO       & posterior predictive & function values, prior      & is one          & maximise acquisition \\
FIBO          & density over optimum & global optimiser, prior     & none            & sample \\
MetaBO, FSAF  & candidate score      & RL reward, solved tasks     & GP, refit       & argmax \\
NAP           & candidate score      & RL reward, solved tasks     & learned jointly & argmax \\
POLAR         & candidate score      & RL reward, solved tasks     & pretrained      & argmax \\
\midrule
\ours{}       & posterior over set   & best candidate, prior       & none or frozen  & argmax \\
\bottomrule
\end{tabular}
\end{table}

\subsection{In-Context Methods}
\label{app:incontext}

\paragraph{Amortised surrogate (PFNs4BO)}
PFNs4BO \citep{muller2023pfns4bo} trains a PFN on synthetic priors to return the posterior predictive $p(y\mid x,\mathcal D)$ in one forward pass. The surrogate is in-context, but the acquisition function is chosen and maximised numerically at every round, and it is restricted to those with a closed form on the PFN's discrete Riemann distribution (Figure~\ref{fig:bo-amortisation}).

\paragraph{Amortised acquisition values (PFNs4BO-KG, $\alpha$-PFN)}
Two methods train a PFN to predict the value of a specified acquisition function. PFNs4BO reports, in its appendix, a PFN trained to approximate the knowledge-gradient acquisition directly. $\alpha$-PFN \citep{rakotoarison2026alphapfn} trains a first PFN conditioned on information about the optimum and a second PFN to predict the expected information gain of entropy search from values measured with the first, replacing the Monte Carlo approximations of entropy search by one forward pass per candidate. Both approaches are in-context. In both cases, however, the acquisition rule is chosen in advance, the network regresses its value, and the query is still obtained by maximising that value over $\mathcal X$.

\paragraph{Amortised optimiser location (FIBO)}
\fibo{} \citep{fibo2025} trains a normalising flow on pairs $(\mathcal D,x^*)$ to model the distribution $p(x^*\mid\mathcal D)$ of the optimiser location, and proposes queries by sampling from it. This amounts to Thompson sampling under the pre-training prior (Figure~\ref{fig:bo-amortisation}). The method is fully in-context, but its decision rule is fixed, it cannot be combined with a surrogate, and its batches consist of independent draws. Pre-training is also costly, since every training pair requires the global optimiser of a prior draw, and the rejection sampling used to make $x^*$ uniform over the domain discards most draws (Appendix~\ref{app:fibo}).

\paragraph{Amortised policies trained by reinforcement learning (NAP, TAMO)}
In NAP \citep{maraval2023nap}, a transformer neural process reads the observations and a candidate set and queries the highest-scoring candidate. It is trained on solved tasks of one family, such as hyperparameter-optimisation datasets for one search space, and transfers only within it. Lacking a label for the query, it learns the score with proximal policy optimisation (PPO) \citep{schulman2017ppo}, a policy-gradient reinforcement learning method, using the regret as the reward, and compensates for the sparsity of this reward with an auxiliary likelihood head. TAMO \citep{zhang2025tamo}, a dimension-agnostic transformer for multi-objective optimisation, maps the observations to the next design and is trained by reinforcement learning on hypervolume improvement. Like \fibo{}, both amortise the BO fully in-context and cannot take a surrogate's information as input.

\paragraph{Learned decision on a pretrained representation (POLAR)}
POLAR \citep{huang2026polar} trains a small head that scores the candidates from the hidden representations of a pretrained tabular foundation model, TabICLv2 \citep{qu2026tabiclv2} by default and TabPFN-2.5 \citep{grinsztajn2025tabpfn25} in an ablation. As in \ours{} with a fixed surrogate, the representation and the decision are separate modules, but the two methods differ in three respects. First, the head is trained by policy gradient on simulated rollouts with one-step improvements as rewards, so its score approximates no stated quantity, whereas our network approximates the posterior probability of Proposition~\ref{prop:posterior-best}. Second, POLAR is trained on the meta-training tasks of each benchmark search space rather than on a synthetic prior. Third, the backbone is not kept fixed. By default, it is fine-tuned on these tasks with an auxiliary prediction loss, so the gains of POLAR over non-adapted baselines combine the learned decision with a task-specific adaptation of the surrogate. Our network, in contrast, is pre-trained only on synthetic priors and either keeps the surrogate frozen or can be trained without one. We do not include POLAR as a baseline, since no public implementation was available at the time of writing.

\subsection{Learned Acquisition Methods}
\label{app:learned-af}

MetaBO, FSAF and BOFormer also learn the query decision, but by reinforcement learning on top of a Gaussian process. At every round, a GP is refitted to the observations, and a network selects a point from a finite candidate set using only summaries of the GP posterior at each candidate, such as its mean and variance. The network is trained to maximise a reward based on the regret or, for multi-objective problems, the hypervolume. Since it never reads the observations, all information about a new task reaches it through the GP fit, so none of these methods is in-context.

\paragraph{MetaBO \citep{volpp2020metabo}}
MetaBO replaces only the acquisition function. A neural network scores each candidate from its location, the GP posterior mean and variance and the remaining budget, and the policy is a softmax over these scores on Sobol points augmented with local maxima of the score. The network is meta-trained with PPO on related source tasks, so that it encodes the structure shared across this family.

\paragraph{FSAF \citep{hsieh2021fsaf}}
FSAF treats the acquisition function as a $Q$-function, parameterised by a Bayesian deep $Q$-network on features of the GP posterior. It is meta-trained with Bayesian model-agnostic meta-learning on a set of training functions, with a demonstration policy based on a standard acquisition function as a prior, and adapted to each new task by a few gradient steps on its first observations. FSAF therefore requires both a GP fit at every round and gradient updates on every new task.

\paragraph{BOFormer \citep{hung2025boformer}}
BOFormer addresses multi-objective BO and learns the acquisition function as a $Q$-function over sequences by non-Markovian deep $Q$-learning, reading the GP posterior rather than the observations.

\subsection{Relation to Our Method}
\label{app:relation}

Like the in-context methods of Appendix~\ref{app:incontext}, \ours{} reads the observations directly, and like the learned acquisition methods of Appendix~\ref{app:learned-af}, it selects the next query from a finite candidate set. It differs from all of them in three respects: how the decision is learned, what it is trained on, and how a surrogate can be used.

First, the decision is learned by supervised classification. Under a synthetic prior, every function can be evaluated at all candidates, so the index of the best candidate is available as a label for every task. The network is trained with cross-entropy on this label, whereas NAP, POLAR, MetaBO, FSAF and BOFormer learn from a reward, and PFNs4BO-KG and $\alpha$-PFN regress a specified acquisition function. Second, training requires only a prior over functions, not a pool of solved tasks as in NAP and POLAR. A trained model therefore applies to any problem of its input dimension and batch size, without meta-training on the target family or adaptation at test time. Third, the same network and loss can be used with or without a fixed surrogate. With a surrogate, only the query selection is amortised, and without one, the whole BO loop runs in-context.

\section{Model Details}
\label{app:model}
This appendix details the model of Section~\ref{sec:method}: the token inputs and attention structure, the network and deployment settings, and the sequential selection of batch points, with its related extended training loss and algorithms.

\subsection{Tokens and Inputs}
\label{app:features}

\paragraph{Surrogate-free model, single query}
Without a surrogate, \ours{} reads one token per observation and one per candidate. Every token carries the point channels of Table~\ref{tab:features}, and observation tokens also carry the value channels. Each token $s$ is embedded as $h^{(0)}_s=W_{\mathrm{feat}}\,t_s$. Observation tokens attend only to each other, and candidate tokens only to the observations. The features of a candidate are carried forward by the residual connections of every layer $\ell$,
\[
\tilde h^{(\ell)}_s=\mathrm{LN}\big(h^{(\ell)}_s+\mathrm{Attn}(h^{(\ell)})_s\big),\qquad
h^{(\ell+1)}_s=\mathrm{LN}\big(\tilde h^{(\ell)}_s+\mathrm{FF}(\tilde h^{(\ell)}_s)\big),
\]
and the head maps the final representation of $c_j$ to $z_j$. The score therefore depends only on $(\mathcal D,c_j)$ and is unchanged by adding, removing or permuting other candidates.

\begin{table}[h]
\vspace{-8pt}
\centering\small
\caption{Token channels. Statistics are over the $n$ observed values, and $x^{+}$ is the incumbent. Channels are zero on tokens that do not carry them. The noise channel holds the task's noise level in pre-training and a fixed $\sigma_0$ at deployment (Appendix~\ref{app:deploycfg}).}
\label{tab:features}
\begin{tabularx}{\textwidth}{@{}lcl>{\raggedright\arraybackslash}X@{}}
\toprule
channel & width & carried by & definition \\
\midrule
coordinates & $d$ & all tokens & $x\in[0,1]^d$ \\
distances & 2 & all tokens & $\min_i\lVert x-x_i\rVert_2$ and $\lVert x-x^{+}\rVert_2$ \\
role & 4 & all tokens & one-hot token type \\
state & 2 & all tokens & $1/\sqrt n$, $\log_{10}\sigma_{\mathrm{noise}}$ \\
\midrule
value & 1 & observations, filled pending & $(y_i-\bar y)/s_y$ with $s_y\ge10^{-3}$, or fantasy $\hat y$ \\
rank & 1 & observations & $\Phi^{-1}\big((r_i+\tfrac12)/n\big)$ \\
gap & 1 & observations & $\tilde y_i-\max_k\tilde y_k$ \\
incumbent & 1 & observations & $\mathbf 1[x_i=x^{+}]$ \\
\bottomrule
\end{tabularx}
\end{table}

\paragraph{With a surrogate}
\ourspfn{} and \oursgp{} use the same tokens, channels and attention mask of \ours, and only the embedding changes to $h^{(0)}_s=W_{\mathrm{feat}}\,t_s+W_\phi\,m_s$, where $m_s$ holds the surrogate features of token $s$ (Table~\ref{tab:surrogate-features}). The surrogate receives only the coordinates and values of the observations and is evaluated before the forward pass of the network. The PFN is run on the context-standardised values, without the power transform that PFNs4BO applies to $y$, and its predictive summaries are not standardised across candidates, since this would make each score depend on the other candidates.

\begin{table}[h]
\vspace{-8pt}
\centering\small
\caption{Surrogate features $m_s$. The surrogate is conditioned on the observations and on the filled pending points at their fantasy values. Empty pending tokens receive zeros.}
\label{tab:surrogate-features}
\begin{tabularx}{\textwidth}{@{}llcl>{\raggedright\arraybackslash}X@{}}
\toprule
model & feature & width & carried by & definition \\
\midrule
\multirow{2}{*}{\ourspfn{}}
 & layer states & 512 & all tokens & $\sum_{\ell}\operatorname{softmax}(\alpha)_\ell\,\mathrm{state}_{\ell,s}$, learned $\alpha\in\mathbb R^{12}$ \\
 & summaries & 5 & candidates & mean, $68\%$ half-width, $10/50/90\%$ quantiles \\
\midrule
\multirow{4}{*}{\oursgp{}}
 & mean & 1 & all tokens & posterior mean $\mu(x)$ \\
 & std.\ dev. & 1 & all tokens & latent posterior $\sigma(x)$ \\
 & correlation & 1 & all tokens & $\operatorname{corr}\big(f(x),f(x^{+})\big)$ \\
 & hyperparameters & $d+2$ & all tokens & log-lengthscales, log-noise, constant mean \\
\bottomrule
\end{tabularx}
\end{table}

\paragraph{Batch selection}
For a batch of $q$ points, every model also carries $P=q-1$ pending tokens, which hold the points already selected in the current round (Appendix~\ref{app:batch}). An empty pending token holds the coordinates of the incumbent and no value. Once a point is selected, the next pending token takes its coordinates and, in its standardised-value channel, a fantasy value $\hat y$. For \ours{}, $\hat y=0$, the mean of the standardised context, and for \ourspfn{} and \oursgp{}, $\hat y$ is the surrogate's posterior mean. Observation and pending tokens attend only to each other, and candidate tokens attend to both, so the score of a candidate becomes $z_j(\mathcal D,\mathcal P,c_j)$, with $\mathcal P$ the pending points. 

With a surrogate, the filled pending points are also passed to the surrogate at their approximated values, and empty pending tokens receive zero surrogate features. The GP is conditioned on the observations and the filled pending points, with the hyperparameters of the round and without refitting. The PFN receives the filled pending points as additional context points.

\subsection{Network and Deployment}
\label{app:architecture}
\label{app:deploycfg}
The encoder of Section~\ref{sec:architecture} is a post-norm transformer with 4 layers of width 512, 8 attention heads, feed-forward width 2048, GELU activations and no dropout. The scoring head consists of Linear$(512,512)$, GELU and Linear$(512,1)$. The network has 12.88M trainable parameters for \ours{}, 12.89M for \oursgp{} and 13.15M for \ourspfn{}, the latter excluding the 26.8M parameters of the frozen PFN. At deployment, every run uses $M=512$ candidates. The noise channel is set to the low end of each prior's noise range, $\sigma_0=10^{-3}$ for models pre-trained on the mixture prior and $\sigma_0=0.025$ for models pre-trained on the HEBO prior.
 
\subsection{Batch Selection}
\label{app:batch}

The points of a batch are selected one at a time through the pending tokens of Appendix~\ref{app:features}. For $q=2$, the loss of one task is
\begin{equation}
\mathcal L=\mathrm{CE}\big(z^{(0)},J^*\big)+\mathbf 1[j_0\neq J^*]\;\mathrm{CE}\big(z^{(1)},J^*\big),\qquad j_0\sim\softmax(z^{(0)}),
\label{eq:batch-loss}
\end{equation}
where $z^{(k)}$ are the scores at slot $k$. The first pick $j_0$ is sampled without gradient and written into a pending token, and the same candidates are scored again. The second term applies only if $j_0$ is not the best candidate, so by Proposition~\ref{prop:posterior-best} it is minimised by the posterior probability that each candidate is the best, given that $j_0$ is not. This mask is what teaches the network to diversify the batch. For $q>2$, the loss of each slot is counted only if no earlier pick is the best candidate, and the losses of $\min(q,4)$ random slots are averaged (Algorithm~\ref{alg:generic-pretrain}). At deployment, a fresh candidate set is drawn for every slot (Algorithm~\ref{alg:generic-infer}). Both algorithms reduce to Algorithms~\ref{alg:pretrain} and~\ref{alg:bo_round} for $q=1$. The models of the batch study in Section~\ref{sec:results-batch} are pre-trained on the HEBO prior.
 
\begin{algorithm}[h]
\caption{Pre-training on one task (surrogate $\phi$ or $\varnothing$, batch size $q$)}
\label{alg:generic-pretrain}
\begin{algorithmic}[1]
\Require prior $p(f)$, surrogate $\phi$ or $\varnothing$, network $\theta$, candidate-set size $M$, batch size $q$
\State Sample $f\sim p(f)$ and a context $\mathcal D$ from $f$
\State Draw $\mathcal C=\textsc{CandidateSet}(\mathcal D,M)$ \Comment{Appendix~\ref{app:candset}}
\State Compute the label $J^*=\argmax_{j\le M} f(c_j)$ \Comment{noise-free $f$}
\State If $\phi$ is a GP, fit its hyperparameters to $\mathcal D$ by MAP
\State Draw $G\subseteq\{0,\dots,q-1\}$ with $|G|=\min(q,4)$ \Comment{slots with a gradient}
\State $\mathcal P\gets\emptyset$, $\mathcal L\gets0$, $m\gets1$
\For{$k=0,\dots,\max G$}
  \State Compute $p^{(k)}_\theta(\cdot\mid\mathcal D,\mathcal P,\mathcal C)$ in one forward pass
  \If{$k\in G$}
    \State $\mathcal L\gets\mathcal L-m\,\log p^{(k)}_\theta(J^*\mid\mathcal D,\mathcal P,\mathcal C)$
  \EndIf
  \State Sample $j_k\sim p^{(k)}_\theta(\cdot\mid\mathcal D,\mathcal P,\mathcal C)$ without gradient, and set $m\gets m\,\mathbf 1[j_k\neq J^*]$
  \State $\mathcal P\gets\mathcal P\cup\{(c_{j_k},\hat y(\phi,c_{j_k}))\}$ \Comment{kriging believer $\hat y$}
\EndFor
\State \Return the task loss $\mathcal L/|G|$ \Comment{averaged over a mini-batch of tasks}
\end{algorithmic}
\end{algorithm}
 
\begin{algorithm}[h]
\caption{One optimisation round (surrogate $\phi$ or $\varnothing$, batch size $q$)}
\label{alg:generic-infer}
\begin{algorithmic}[1]
\Require observations $\mathcal D$, surrogate $\phi$ or $\varnothing$, network $\theta$, candidate-set size $M$, batch size $q$
\State If $\phi$ is a GP, fit its hyperparameters to $\mathcal D$ by MAP
\State $\mathcal P\gets\emptyset$
\For{$k=0,\dots,q-1$}
  \State Draw $\mathcal C_k=\textsc{CandidateSet}(\mathcal D,M)$ 
  \State Compute $p^{(k)}_\theta(\cdot\mid\mathcal D,\mathcal P,\mathcal C_k)$ in one forward pass
  \State $x_k\gets c_{\hat\jmath}\in\mathcal C_k$ with $\hat\jmath=\argmax_{j\le M}p^{(k)}_\theta(j\mid\mathcal D,\mathcal P,\mathcal C_k)$
  \State $\mathcal P\gets\mathcal P\cup\{(x_k,\hat y(\phi,x_k))\}$ \Comment{kriging believer $\hat y$}
\EndFor
\State Evaluate $f$ at $x_0,\dots,x_{q-1}$, observe $y_0,\dots,y_{q-1}$, and add them to $\mathcal D$
\end{algorithmic}
\end{algorithm}

\section{Pre-Training Details}
\label{app:datagen}

This appendix shows how the pre-training tasks of Section~\ref{sec:objective} are generated, and provides details for the two pre-training priors (Appendix~\ref{app:prior}), the candidate set sampler (Appendix~\ref{app:candset}), pre-training settings and computational costs (Appendix~\ref{app:training}).

\subsection{Priors}
\label{app:prior}

\paragraph{HEBO prior}
\citet{muller2023pfns4bo} propose several pre-training priors for PFNs4BO, among them a GP-based prior inspired by the surrogate of HEBO \citep{cowenrivers2022hebo}. We pre-train on the HEBO prior, with the kernel and hyperpriors of the released \pfnhebo{} model (Table~\ref{tab:hebo-prior}). A GP path cannot be drawn on a fixed grid in advance, since the context points are chosen adaptively from earlier values. The value of $f$ at each new point, including the candidates for the label, is therefore drawn from the GP conditioned on all values drawn so far. This gives the same joint distribution as a single draw of $f$ at all these points.

\begin{table}[h]
\vspace{-8pt}
\centering\small
\caption{HEBO prior. Functions are drawn as $f\sim\mathcal{GP}(0,s^2k_{3/2})$, with $k_{3/2}$ the Mat\'ern-3/2 ARD kernel with lengthscales $\ell$. Gamma distributions are given by shape and rate.}
\label{tab:hebo-prior}
\begin{tabularx}{\textwidth}{@{}l>{\raggedright\arraybackslash}X@{}}
\toprule
lengthscales & $\ell_i\sim\mathrm{Gamma}(1.2107,1.5212)$ \\
output scale & $s^2\sim\mathrm{Gamma}(0.8452,0.3994)$ \\
noise & $\sigma^2\sim\mathrm{LogNormal}(-4.63,0.5)$, with relative noise $\sqrt{\sigma^2/s^2}$ clamped to $[10^{-3},10]$ \\
active dimensions & each dimension dropped with probability $0.3$, keeping at least one \\
\bottomrule
\end{tabularx}
\end{table}

\begin{table}[t!]
\vspace{-8pt}
\centering\fontsize{11}{13}\selectfont
\setlength{\tabcolsep}{4pt}
\setlength{\abovecaptionskip}{2pt}
\setlength{\belowcaptionskip}{2pt}
\renewcommand{\arraystretch}{1.0}
\caption{Definition of the mixture prior. The warped family uses Kumaraswamy CDFs $\kappa_i$ \citep{kumaraswamy1980generalized,snoek2014warping}, with $\alpha_i,\beta_i$ clamped to $[0.2,5]$. In the bumps family, the division by $w_k$ is elementwise, and $\mathrm{sigm}$ denotes the logistic function.}
\label{tab:prior}
\begin{tabularx}{\textwidth}{@{}lc>{\raggedright\arraybackslash}X@{}}
\toprule
family & weight & definition \\
\midrule
stationary & $0.3$ & $\mathrm{rff}$ \\
warped & $\approx0.233$ & $\mathrm{rff}\circ\kappa$,\; $\kappa_i(x)=1-(1-x^{\beta_i})^{\alpha_i}$ \newline $\alpha_i,\beta_i\sim\mathrm{LogNormal}(0,0.75)$ \\
bumps & $\approx0.233$ & $\sum_{k\le K}h_k\exp\big(-\tfrac12\lVert(x-c_k)/w_k\rVert^2\big)$ \newline $K\sim\mathcal U\{2,\dots,10\}$,\; $c_k\sim\mathcal U[0,1]^d$,\; $w_{ki}\sim\mathrm{LogU}(0.03,2)$ \\
plateau & $\approx0.233$ & $\sum_{j\le L}\mathrm{sigm}\big((\mathrm{rff}-t_j)/\tau\big)+10^{-3}\,\mathrm{rff}$ \newline $L\sim\mathcal U\{2,\dots,7\}$,\; $\tau\sim\mathrm{LogU}(0.01,0.1)$ \\
\bottomrule
\end{tabularx}
\vspace{-6pt}
\end{table}

\paragraph{Mixture prior}
The mixture prior extends pre-training to functions that stationary GP surrogates model poorly. Each function is drawn from one of four families, with the weights of Table~\ref{tab:prior}. The stationary family covers smooth functions, on which GP surrogates are well specified. The warped family applies a monotone transformation to each input, so that the effective lengthscale varies across the domain \citep{snoek2014warping}. The bumps family sums a few Gaussian basins, which produces narrow optima within nearly flat regions. The plateau family applies steep sigmoids to a stationary path, which produces piecewise-flat functions with sharp steps, such as those caused by thresholds or regime changes in physical processes.

The stationary, warped and plateau families are built on the random-Fourier-feature approximation of a GP path \citep{rahimi2007random}, as in the pre-training prior of \fibo{} \citep{fibo2025}, which we extend from the squared-exponential kernel to Mat\'ern-$\nu$ kernels,
\[
\mathrm{rff}(x)=\sqrt{2/m}\,\textstyle\sum_{k=1}^{m}a_k\cos(\omega_k^\top x+b_k),\qquad m=256,
\]
where $a_k\sim\mathcal N(0,1)$, $b_k\sim\mathcal U(0,2\pi)$, $\omega_k$ is drawn from the Mat\'ern-$\nu$ spectral density with $\nu\sim\mathcal U\{\tfrac12,\tfrac32,\tfrac52,\infty\}$, and the lengthscales follow the HEBO prior, floored at $0.02$. Every function is standardised over a 1024-point Sobol grid, and observations carry Gaussian noise with $\sigma\sim\mathrm{LogU}(10^{-3},0.1)$.

\subsection{Candidate Sampling}
\label{app:candset}

Candidates are drawn by Algorithm~\ref{alg:candset}, identically in pre-training and at deployment. They consist of $40\%$ scrambled Sobol points, $30\%$ uniform points and $30\%$ Gaussian perturbations of the three best observations, clipped to $[0,1]^d$ by the projection $\Pi$. The space-filling points support exploration, and the perturbations allow the network to refine the incumbent. Their scale is drawn log-uniformly over an order of magnitude, which covers both coarse and fine moves without a fixed trust-region radius. Pre-training draws $M\in\{64,128,\dots,384\}$, while all experiments use $M=512$. This is possible because each score is independent of the other candidates (Appendix~\ref{app:features}).

\begin{algorithm}[h]
\caption{$\textsc{CandidateSet}(\mathcal D,M)$}
\label{alg:candset}
\begin{algorithmic}[1]
\State $S\gets 0.4M$ scrambled-Sobol points in $[0,1]^d$
\State $U\gets 0.3M$ uniform points in $[0,1]^d$
\State $A\gets$ the inputs of the three observations in $\mathcal D$ with the largest $y$
\For{$t=1,\dots,0.3M$}
  \State Draw $a\sim\mathcal U(A)$, $s\sim\mathrm{LogU}(0.02,0.2)$ and $\xi\sim\mathcal N(0,I_d)$, and set $l_t\gets\Pi(a+s\,\xi)$
\EndFor
\State \Return $S\cup U\cup\{l_t\}_{t\le0.3M}$
\end{algorithmic}
\end{algorithm}

\subsection{Training Configurations and Cost}
\label{app:training}
All \ours{} models share the settings of Table~\ref{tab:recipe}, and all models were trained on a single NVIDIA H100 GPU. PFNs4BO, \fibo{} and NAP use their own recipes (Appendix~\ref{app:baselines}). Table~\ref{tab:checkpoints} reports the training time as the median time per epoch, excluding the first epoch, multiplied by the number of epochs, without evaluation and data generation. For NAP, which is trained by PPO rather than in epochs, we report the time of the full training run, and its 3D and 4D models use different configurations (Appendix~\ref{app:nap}), which explains the difference in their training times.

\begin{table}[h]
\centering\footnotesize
\setlength{\abovecaptionskip}{2pt}
\setlength{\belowcaptionskip}{3pt}
\setlength{\tabcolsep}{4pt}
\renewcommand{\arraystretch}{0.9}
\begin{minipage}[t]{0.47\textwidth}
\vspace{0pt}
\centering
\caption{Pre-training settings.}
\label{tab:recipe}
\begin{tabular}{@{}ll@{}}
\toprule
setting & value \\
\midrule
optimiser & AdamW \\
learning rate & $10^{-4}$, cosine, 5 warm-up epochs \\
weight decay & $10^{-2}$ \\
gradient clipping & $1.0$ \\
epochs & 100 \\
steps per epoch & 1024 \\
batch size & 128 (2 micro-batches of 64) \\
optimiser steps & $51{,}200$ \\
\bottomrule
\end{tabular}
\end{minipage}\hspace{0.015\textwidth}
\begin{minipage}[t]{0.44\textwidth}
\vspace{0pt}
\centering
\caption{Training time (H100 hours), $q=2$.}
\label{tab:checkpoints}
\begin{tabular}{@{}llcc@{}}
\toprule
model & prior & $d=3$ & $d=4$ \\
\midrule
\ours{}    & HEBO    & 2.73 & 2.74 \\
           & mixture & 4.37 & 4.28 \\
\ourspfn{} & HEBO    & 4.63 & 4.59 \\
           & mixture & 6.06 & 6.10 \\
\oursgp{}  & HEBO    & 2.65 & 2.66 \\
           & mixture & 2.65 & 2.67 \\
\midrule
PFNs4BO    & mixture  & 0.74 & 0.74 \\
\fibo{}    & GP prior & 1.78 & 1.94 \\
           & mixture  & 1.94 & 2.00 \\
NAP        & GP prior & 6.77 & 11.1 \\
\bottomrule
\end{tabular}
\end{minipage}
\end{table}

\section{Baselines}
\label{app:baselines}

\subsection{GP-MAP and GP-FB}
\label{app:gpsbaselines}

Both GP baselines are implemented in BoTorch~0.18.1 \citep{balandat2020botorch}, in double precision. They use a Mat\'ern-5/2 ARD kernel and a homoscedastic Gaussian likelihood, with inputs in the unit cube and standardised outputs. At every round, the next batch is chosen by maximising \texttt{qLogNoisyExpectedImprovement} \citep{ament2023logei}, estimated with 256 quasi-Monte Carlo samples, jointly over the $q$ points with \texttt{optimize\_acqf}, using 512 raw samples and 10 L-BFGS-B restarts. The two baselines differ only in how they handle the hyperparameters. \gpmap{} uses \texttt{SingleTaskGP} with BoTorch's dimension-scaled log-normal hyperpriors and fits the hyperparameters by MAP at every round. \gpfb{} uses \texttt{FullyBayesianSingleTaskGP}, samples the hyperparameter posterior with NUTS (256 warm-up steps, 8 retained samples), and averages the acquisition function over these samples.

\subsection{PFNs4BO}
\label{app:pfns4bo}

We use the two checkpoints released by \citet{muller2023pfns4bo}.\footnote{\url{https://github.com/automl/PFNs4BO}} \pfnhebo{} (12 layers) is pre-trained on the HEBO prior (Appendix~\ref{app:prior}), and \pfnbnn{} (6 layers) on a prior of random Bayesian neural networks with warped inputs. At every round, the observed values are Yeo--Johnson transformed, and expected improvement is computed in closed form from the predictive distribution and maximised with the authors' optimiser, which refines the best 100 of $100{,}000$ uniform samples by L-BFGS-B. The released code proposes one point per round and does not support batch optimisation, so we build batches with the kriging-believer rule \citep{ginsbourger2010kriging}. Each selected point is added to the observations with its predictive mean as value, and the next point maximises expected improvement on the augmented observations.

\subsection{FIBO}
\label{app:fibo}

\fibo{} \citep{fibo2025} pre-trains a conditional normalising flow $p_\vartheta(x^*\mid\mathcal D)$ of the optimiser location on pairs $(\mathcal D,x^*)$ from a GP prior, and proposes a batch as $q$ independent samples. No implementation is public, so we reimplemented it following the specifications mentioned in the paper. One model is trained per dimension. The encoder is the PFNs4BO transformer pre-trained on the BNN prior, whose outputs are averaged over the data points and projected to a context vector of dimension 256 ($d=3$) or 512 ($d=4$). The generative head is a normalising flow, which maps a simple base distribution through a sequence of invertible transformations to a density over the domain, conditioned on the context vector. Following the paper, it is an autoregressive neural spline flow \citep{durkan2019nsf}, in which each block transforms one coordinate at a time with a monotone rational-quadratic spline whose parameters depend on the previous coordinates and on the context vector. It has 6 ($d=3$) or 8 ($d=4$) blocks, each with 6 layers of 256 hidden units. Functions are random-Fourier-feature paths of a squared-exponential GP with the paper's hyperpriors, the optimum is located by multi-start L-BFGS-B, and rejection sampling makes the marginal of $x^*$ uniform. Training uses $70{,}000$ pairs, Adam with learning rate $10^{-4}$, a cosine schedule, 200 epochs and batch size 128. In our runs, rejection sampling discards about $99\%$ of the draws, with $5.6$M optimisations at $d=3$ and $18.2$M at $d=4$.

Three choices depart from the paper, each intended to favour \fibo{}. The base distribution of the flow is uniform on the domain, so that samples are always feasible, the observed values are standardised over the context instead of Yeo--Johnson transformed, and the context sizes cover those of our protocols. On functions drawn from its prior at $d=3$, our implementation reaches a GAP of $0.90$ after 58 evaluations, against $0.93$ after 200 in the paper, and it beats random search on every test function.

\subsection{NAP}
\label{app:nap}
NAP \citep{maraval2023nap} is a transformer neural process that reads the observations and a candidate set, scores every candidate and queries the one with the highest score. We use the authors' code,\footnote{\url{https://github.com/huawei-noah/HEBO/tree/master/NAP}} with 6 layers of width 512 and a bar-distribution head with 1000 buckets, and, since it provides no training script for continuous functions, we add one that changes only the task environment. NAP is meta-trained with PPO for 2000 iterations on the GP task family of MetaBO \citep{volpp2020metabo}, that is, on Mat\'ern-5/2 GP draws with lengthscales in $[0.05,0.5]$, unit signal variance and noise variance $0.1$, with a regret-based reward and an auxiliary loss that fits the predictive head to the observed values. In 4D, we use the published configuration, with episodes of 30 steps without initial points and 2000 Sobol candidates per query. In 3D, we adapt it to our protocol, with episodes of 90 steps from 10 initial points, observed values standardised within the context and the candidate sets of the other candidate-based methods (Appendix~\ref{app:candset}). NAP skips candidates that coincide with an observed or already selected point, and for $q=2$ it imputes the value of its first pick with the mean of its own predictive head, adds it to the context and then selects the second point.

\section{Benchmark Problems}
\label{app:benchmarks}

\subsection{Problems}
\label{app:problem-sets}

Table~\ref{tab:problems} lists the ten 3D and five 4D problems. Every problem is mapped to the unit cube and maximised, with minimisation problems negated.

\paragraph{Analytic functions}
The analytic functions use the implementations of BoTorch~0.18.1 \citep{balandat2020botorch},\footnote{\url{https://github.com/pytorch/botorch}, module \texttt{botorch.test\_functions}.} whose definitions follow the Virtual Library of Simulation Experiments \citep{surjanovic2013virtual}. They cover needle-in-a-haystack landscapes (Ackley), dense ripples on a global trend (Levy, Rastrigin), narrow ridges on a plateau (Michalewicz), an ill-conditioned valley (Rosenbrock) and smooth surfaces with few optima (Hartmann, Six-Hump Camel). Rosenbrock uses the box $[-2.048,2.048]^d$ instead of BoTorch's default $[-5,10]^d$. The two-dimensional Six-Hump Camel function is embedded in 3D through a third coordinate that does not affect $f$.

\paragraph{Bioprocess surface}
This surface, constructed for this work, models the tuning of a microbial fuel cell over three concentrations with Haldane substrate-inhibition kinetics \citep{haldane1930enzymes},
\[
f(x)=h(x_1;p_1)\,h\big(x_2;p_2(x_3)\big)\,h(x_3;p_3),\qquad p_2(x_3)=p_2\,\big(1+0.8\,(x_3-0.5)\big),
\]
where $h(\cdot;p)$ is the Haldane rate $v(s)=s/(K+s+s^2/K_i)$ normalised to one at its peak $p$, with $(p_1,p_2,p_3)=(0.10,0.20,0.35)$ and $\sqrt{K/K_i}=6$. The dependence of $p_2$ on $x_3$ makes it non-separable. It is unimodal but strongly non-stationary, and about $0.3\%$ of the domain lies within $10\%$ of the optimum $f^*=1$ at $x^*=(0.10,0.176,0.35)$.

\paragraph{Olympus campaigns}
The real-world campaigns come from Olympus \citep{hase2021olympus},\footnote{\url{https://github.com/aspuru-guzik-group/olympus}, package \texttt{olymp} on PyPI.} whose neural-network emulators, trained on data from automated experiments, serve as oracles. We use the deterministic \texttt{NeuralNet} emulators, so the optimum can be pinned and regret computed exactly. Since the package requires TensorFlow~1.15, we extracted the emulator weights from \texttt{olymp}~0.0.1b0 and reimplemented the forward pass in PyTorch, which reproduces the stored test $R^2$ of the 3D emulators to seven significant figures. \emph{Fullerenes} maximises the mole fraction of a bis-adduct in a flow reactor, on data from \citet{walker2017tuning}, and \emph{colours} (\texttt{colors\_n9}) minimises the distance of a mixed colour to a target green (Table~\ref{tab:olympus}). \emph{Benzylation} minimises the impurity of an N-benzylation reaction in 4D. Its response surface is smooth, which makes it well suited to GP modelling.

\paragraph{GP-prior draws}
The 4D GP-prior problem consists of ten functions drawn from our reimplementation of the GP prior of \fibo{}, a random-Fourier-feature squared-exponential GP (Appendix~\ref{app:fibo}). Each is optimised as a separate problem, and their results are averaged into one. They are problems on which GP surrogates are well specified by construction.

\begin{table}[h]
\vspace{-6pt}
\centering\footnotesize
\setlength{\tabcolsep}{3pt}
\setlength{\aboverulesep}{0pt}
\setlength{\belowrulesep}{0pt}
\renewcommand{\arraystretch}{1.15}
\caption{Benchmark problems. Minimisation problems are negated.}
\label{tab:problems}
\begin{tabular}{@{}lcll@{\hspace{0.6em}}|@{\hspace{0.6em}}lcll@{}}
\toprule
problem & $d$ & native domain & source & problem & $d$ & native domain & source \\
\midrule
Ackley & 3, 4 & $[-32.768,32.768]^d$ & BoTorch & Six-Hump Camel & 3 & $[-3,3]\times[-2,2]$ & BoTorch \\
Levy & 3, 4 & $[-10,10]^d$ & BoTorch & bioprocess & 3 & $[0,1]^3$ & this work \\
Rosenbrock & 3, 4 & $[-2.048,2.048]^d$ & BoTorch & fullerenes & 3 & emulator & Olympus \\
Rastrigin & 3 & $[-5.12,5.12]^3$ & BoTorch & colours & 3 & emulator & Olympus \\
Michalewicz & 3 & $[0,\pi]^3$ & BoTorch & benzylation & 4 & emulator & Olympus \\
Hartmann & 3 & $[0,1]^3$ & BoTorch & GP-prior draws & 4 & $[0,1]^4$ & \fibo{} prior \\
\bottomrule
\end{tabular}
\vspace{-6pt}
\end{table}

\subsection{Optima}
\label{app:pinned}

Regret is computed against recomputed global optima, since rounded tabulated constants, such as that of Hartmann in BoTorch, which lies $2\times10^{-6}$ below the true maximum, would put a floor under the regret. Closed-form optima are used exactly, and all others were found by multi-start L-BFGS-B and confirmed by an independent second search (Tables~\ref{tab:synthetic} and~\ref{tab:olympus}, on the unit cube). In 4D, Ackley, Levy and Rosenbrock have their optima at the same coordinates as in 3D, and benzylation at the corner $x^*=(1,0,0,0)$, with $f^*=-2.33504673$. The GP-prior draws are filtered as in the pre-training of \fibo{}, so that their optima lie inside the cube and are spread uniformly, use seeds outside those of \fibo{}'s training data, and have their optima checked against $2\times10^{6}$ Sobol points.

\begin{table}[h]
\vspace{-6pt}
\centering\footnotesize
\setlength{\tabcolsep}{2.5pt}
\setlength{\belowcaptionskip}{2pt}
\begin{minipage}[t]{0.53\textwidth}
\centering
\caption{Analytic functions in 3D, with $\cdot$ any value of the inert coordinate.}
\label{tab:synthetic}
\begin{tabular}{@{}lll@{}}
\toprule
problem & $x^*$ & $f^*$ \\
\midrule
Ackley & $(0.5,0.5,0.5)$ & $0$ \\
Levy & $(0.55,0.55,0.55)$ & $0$ \\
Rastrigin & $(0.5,0.5,0.5)$ & $0$ \\
Rosenbrock & $(0.7441,0.7441,0.7441)$ & $0$ \\
Michalewicz & $(0.7012,0.5,0.4090)$ & $2.76039468$ \\
Hartmann & $(0.1146,0.5556,0.8525)$ & $3.86278$ \\
Six-Hump Camel & $(0.515,0.322,\cdot)$ & $1.03162845$ \\
               & $(0.485,0.678,\cdot)$ & \\
\bottomrule
\end{tabular}
\end{minipage}\hfill
\begin{minipage}[t]{0.415\textwidth}
\centering
\caption{Olympus campaigns in 3D, with $n$ measurements and test $R^2$ per emulator.}
\label{tab:olympus}
\begin{tabular}{@{}lll@{}}
\toprule
 & fullerenes & colours \\
\midrule
$n$ & 246 & 102 \\
$R^2$ & 0.990 & 0.940 \\
$x^*$ & $(0.2693,0.4361,0.0513)$ & $(0,1,0.8075)$ \\
$f^*$ & $0.95243518$ & $-0.06969733$ \\
\bottomrule
\end{tabular}
\end{minipage}
\vspace{-6pt}
\end{table}

\section{Additional Results and Ablations}
\label{app:results}
This appendix reports more detailed results behind Section~\ref{sec:results}, including per-problem results, pairwise win rates and results at every context size, and adds analyses not shown in the main text. It covers the 3D and 4D problems (Appendix~\ref{app:full3d}), the comparison with expected improvement (Appendix~\ref{app:controlled}), the effect of the pre-training prior (Appendix~\ref{app:prior-results}), batch optimisation (Appendix~\ref{app:bobatch}) and additional ablations (Appendix~\ref{app:ablations}).

\subsection{Additional Results on the 3D and 4D Problems}
\label{app:full3d}
\label{app:perproblem}
\label{app:4d}

This section is an extension of Section~\ref{sec:results-main}. For the 3D problems, Figures~\ref{fig:main-n10} and~\ref{fig:main-n20} repeat Figure~\ref{fig:main} at $n_{\text{init}}=10$ and $20$, and Table~\ref{tab:winrates3d} complements the ranks of Table~\ref{tab:rank} with the paired win rate of every \ours{} variant against every baseline. For the five 4D problems (Appendix~\ref{app:benchmarks}), whose ranks appear in Table~\ref{tab:rank}, Figure~\ref{fig:4d} shows the regret and rank curves. For both dimensions, Figures~\ref{fig:perproblem} and~\ref{fig:perproblem-4d} and Table~\ref{tab:4d-final} give the results per problem.

\paragraph{Smaller context sizes and paired win rates}
At $n_{\text{init}}=10$ and $20$, the ordering of the methods is the same as at $n_{\text{init}}=40$. Pre-trained on the mixture prior, each \ours{} variant beats \gpmap{} and \gpfb{} in about $80\%$ of runs and every in-context baseline in more than $83\%$. Pre-trained on the HEBO prior, \ours{} and \oursgp{} still beat \gpmap{} in $65$ to $75\%$ of runs, whereas \ourspfn{} ties with it at $n_{\text{init}}=10$ and only pulls ahead with more context.

\paragraph{Per-problem results}
In 3D, the gains of \ours{} are largest on Ackley, Camel, Levy and the Olympus colours campaign, where it reaches a lower regret than all baselines early in the run and keeps it until the end (Figure~\ref{fig:perproblem}). The exception is Rosenbrock, a smooth unimodal valley on which a Mat\'ern GP is well specified, and where the GP baselines end with lower regret. In 4D, the \ours{} variants pre-trained on the mixture prior again reach a lower final regret than \gpmap{} on Ackley, Levy and Rosenbrock (Table~\ref{tab:4d-final}, Figure~\ref{fig:perproblem-4d}). The GP-prior draws and benzylation, which were included as problems on which a GP is well specified, favour the GP baselines, for two different reasons. On the GP-prior draws, the objective is a GP sample, so a GP surrogate can approximate it almost exactly by construction. \oursgp{}, which conditions on the same GP, comes closest to \gpmap{} and improves on it in $55$ to $66\%$ of paired runs, whereas the variants without a GP must infer this structure from the observations alone, and perform worse. On benzylation, the response surface is smooth, and gradient-based maximisation of the acquisition function lets \gpmap{} and \gpfb{} place a query exactly at the optimum, whereas the regret of \ours{} and all other in baselines stalls at a small non-zero value once it has found the optimum region. 

\begin{figure}[t!]
\centering
\setlength{\abovecaptionskip}{1pt}
\setlength{\belowcaptionskip}{0pt}
\includegraphics[width=0.98\linewidth]{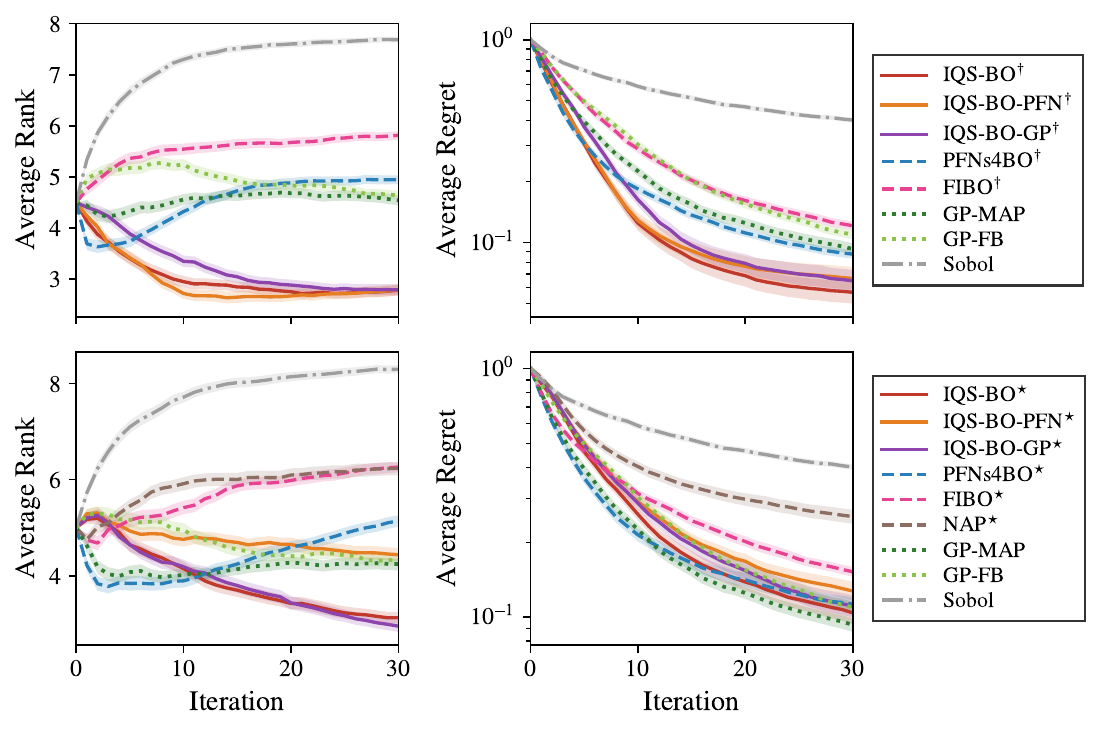}
\caption{As Figure~\ref{fig:main}, at $n_{\text{init}}=10$.}
\label{fig:main-n10}

\vspace{1pt}
\includegraphics[width=0.98\linewidth]{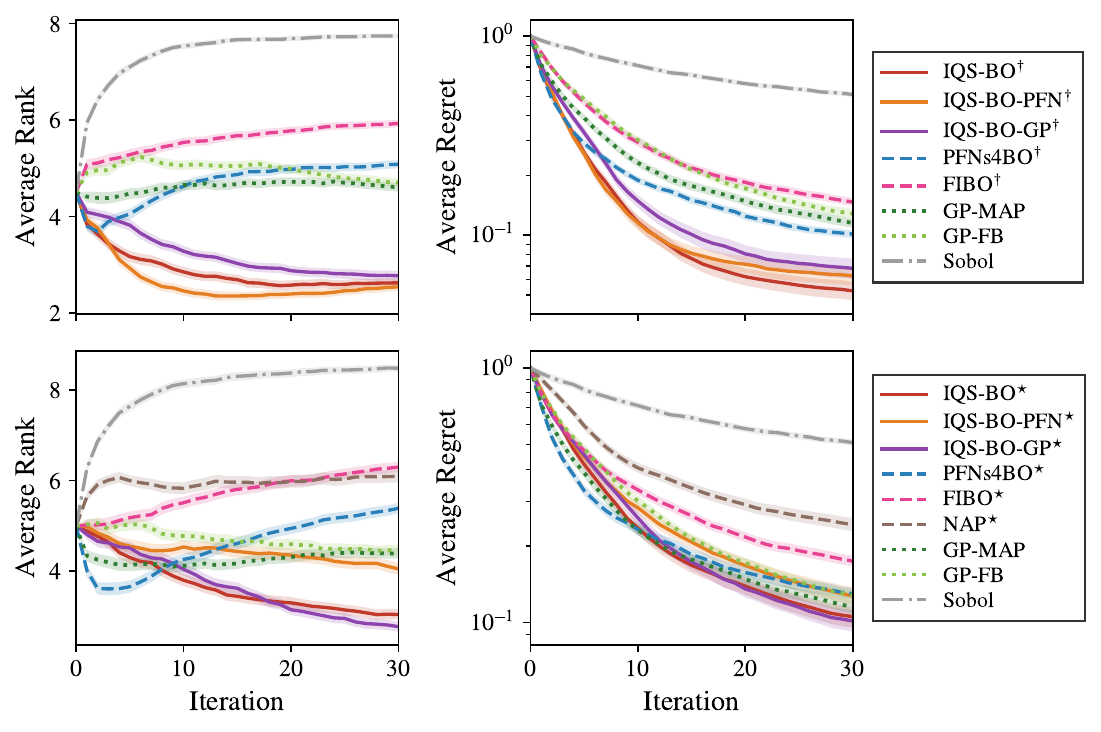}
\caption{As Figure~\ref{fig:main}, at $n_{\text{init}}=20$.}
\label{fig:main-n20}
\end{figure}

\begin{table}[t!]
\centering\footnotesize
\setlength{\tabcolsep}{2pt}
\caption{Paired win rate (\%) on $R^{\text{final}}$ of each \ours{} variant against each baseline, ten 3D problems, 900 paired runs per cell, $\pm$ $95\%$ bootstrap CI half-width. $^{\dagger}$ mixture prior, $^{\star}$ HEBO prior.}
\label{tab:winrates3d}
\vspace{4pt}
\begin{tabular}{@{}cl*{6}{>{\centering\arraybackslash}p{0.118\linewidth}}@{}}
\toprule
 & model & \gpmap & \gpfb & \pfnhebo & \pfnbnn & \fibo & Sobol \\
\midrule
\multirow{6}{*}{\rotatebox[origin=c]{90}{$n_{\text{init}}=10$}} & \ours$^{\dagger}$ & $77.4{\scriptstyle\pm2.4}$ & $77.3{\scriptstyle\pm2.4}$ & $86.4{\scriptstyle\pm2.1}$ & $85.1{\scriptstyle\pm2.2}$ & $92.7{\scriptstyle\pm1.7}$ & $96.4{\scriptstyle\pm1.2}$ \\
 & \ourspfn$^{\dagger}$ & $76.1{\scriptstyle\pm2.6}$ & $77.3{\scriptstyle\pm2.4}$ & $86.3{\scriptstyle\pm2.2}$ & $85.7{\scriptstyle\pm2.2}$ & $90.9{\scriptstyle\pm1.9}$ & $95.4{\scriptstyle\pm1.4}$ \\
 & \oursgp$^{\dagger}$ & $77.6{\scriptstyle\pm2.6}$ & $78.8{\scriptstyle\pm2.4}$ & $85.7{\scriptstyle\pm2.1}$ & $83.9{\scriptstyle\pm2.2}$ & $91.0{\scriptstyle\pm1.8}$ & $96.3{\scriptstyle\pm1.2}$ \\
 & \ours$^{\star}$ & $65.2{\scriptstyle\pm2.7}$ & $67.4{\scriptstyle\pm2.7}$ & $75.3{\scriptstyle\pm2.7}$ & $75.4{\scriptstyle\pm2.8}$ & $83.6{\scriptstyle\pm2.2}$ & $94.2{\scriptstyle\pm1.5}$ \\
 & \ourspfn$^{\star}$ & $49.7{\scriptstyle\pm3.1}$ & $51.0{\scriptstyle\pm3.0}$ & $60.4{\scriptstyle\pm2.9}$ & $61.8{\scriptstyle\pm3.1}$ & $72.0{\scriptstyle\pm2.8}$ & $90.9{\scriptstyle\pm1.8}$ \\
 & \oursgp$^{\star}$ & $71.0{\scriptstyle\pm2.8}$ & $71.6{\scriptstyle\pm2.7}$ & $77.0{\scriptstyle\pm2.6}$ & $76.7{\scriptstyle\pm2.6}$ & $85.3{\scriptstyle\pm2.1}$ & $93.7{\scriptstyle\pm1.6}$ \\
\midrule
\multirow{6}{*}{\rotatebox[origin=c]{90}{$n_{\text{init}}=20$}} & \ours$^{\dagger}$ & $79.7{\scriptstyle\pm2.2}$ & $80.3{\scriptstyle\pm2.2}$ & $90.1{\scriptstyle\pm1.8}$ & $86.4{\scriptstyle\pm2.2}$ & $94.4{\scriptstyle\pm1.4}$ & $97.7{\scriptstyle\pm0.9}$ \\
 & \ourspfn$^{\dagger}$ & $81.6{\scriptstyle\pm2.3}$ & $81.6{\scriptstyle\pm2.2}$ & $90.7{\scriptstyle\pm1.7}$ & $86.1{\scriptstyle\pm2.1}$ & $94.6{\scriptstyle\pm1.4}$ & $97.6{\scriptstyle\pm0.9}$ \\
 & \oursgp$^{\dagger}$ & $78.3{\scriptstyle\pm2.4}$ & $79.9{\scriptstyle\pm2.3}$ & $86.7{\scriptstyle\pm2.1}$ & $83.1{\scriptstyle\pm2.2}$ & $92.4{\scriptstyle\pm1.7}$ & $98.0{\scriptstyle\pm0.9}$ \\
 & \ours$^{\star}$ & $69.6{\scriptstyle\pm2.6}$ & $70.1{\scriptstyle\pm2.6}$ & $79.5{\scriptstyle\pm2.5}$ & $75.6{\scriptstyle\pm2.7}$ & $85.2{\scriptstyle\pm2.2}$ & $95.8{\scriptstyle\pm1.3}$ \\
 & \ourspfn$^{\star}$ & $57.3{\scriptstyle\pm3.2}$ & $57.1{\scriptstyle\pm3.1}$ & $67.2{\scriptstyle\pm2.8}$ & $62.6{\scriptstyle\pm2.9}$ & $76.5{\scriptstyle\pm2.5}$ & $94.3{\scriptstyle\pm1.5}$ \\
 & \oursgp$^{\star}$ & $75.0{\scriptstyle\pm2.8}$ & $76.7{\scriptstyle\pm2.6}$ & $80.7{\scriptstyle\pm2.3}$ & $74.7{\scriptstyle\pm2.6}$ & $87.6{\scriptstyle\pm2.0}$ & $96.1{\scriptstyle\pm1.2}$ \\
\midrule
\multirow{6}{*}{\rotatebox[origin=c]{90}{$n_{\text{init}}=40$}} & \ours$^{\dagger}$ & $79.6{\scriptstyle\pm2.2}$ & $80.6{\scriptstyle\pm2.1}$ & $92.4{\scriptstyle\pm1.6}$ & $91.6{\scriptstyle\pm1.7}$ & $94.2{\scriptstyle\pm1.5}$ & $98.8{\scriptstyle\pm0.7}$ \\
 & \ourspfn$^{\dagger}$ & $81.2{\scriptstyle\pm2.1}$ & $81.7{\scriptstyle\pm2.1}$ & $92.4{\scriptstyle\pm1.6}$ & $90.0{\scriptstyle\pm1.8}$ & $94.7{\scriptstyle\pm1.3}$ & $98.3{\scriptstyle\pm0.8}$ \\
 & \oursgp$^{\dagger}$ & $81.1{\scriptstyle\pm2.3}$ & $82.7{\scriptstyle\pm2.3}$ & $91.6{\scriptstyle\pm1.7}$ & $87.9{\scriptstyle\pm1.8}$ & $94.3{\scriptstyle\pm1.4}$ & $98.3{\scriptstyle\pm0.8}$ \\
 & \ours$^{\star}$ & $68.6{\scriptstyle\pm2.6}$ & $71.9{\scriptstyle\pm2.5}$ & $82.5{\scriptstyle\pm2.4}$ & $80.3{\scriptstyle\pm2.4}$ & $87.8{\scriptstyle\pm2.0}$ & $96.2{\scriptstyle\pm1.2}$ \\
 & \ourspfn$^{\star}$ & $58.6{\scriptstyle\pm2.9}$ & $62.6{\scriptstyle\pm2.9}$ & $75.2{\scriptstyle\pm2.5}$ & $68.2{\scriptstyle\pm2.6}$ & $80.6{\scriptstyle\pm2.3}$ & $95.6{\scriptstyle\pm1.3}$ \\
 & \oursgp$^{\star}$ & $70.1{\scriptstyle\pm2.9}$ & $72.6{\scriptstyle\pm2.8}$ & $79.9{\scriptstyle\pm2.4}$ & $77.5{\scriptstyle\pm2.2}$ & $87.3{\scriptstyle\pm2.0}$ & $96.8{\scriptstyle\pm1.1}$ \\
\bottomrule
\end{tabular}

\end{table}

\begin{figure}[p]
\centering
\includegraphics[width=\linewidth,height=0.95\textheight,keepaspectratio]{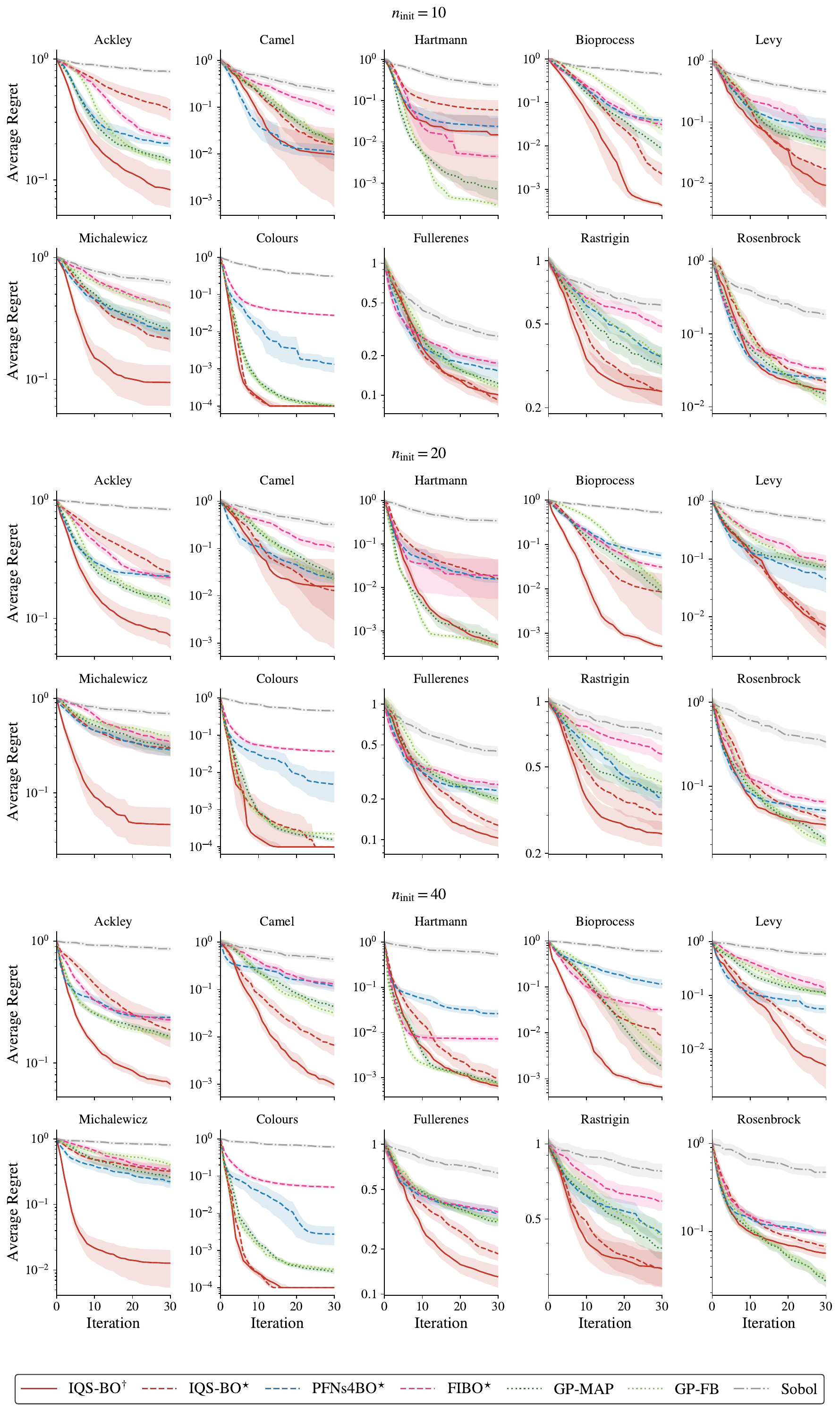}
\caption{Average regret per iteration on each 3D problem, $n_{\text{init}}=10$ (top), $20$ (middle) and $40$ (bottom), 90 seeds, with $95\%$ confidence intervals.}
\label{fig:perproblem}
\end{figure}

\begin{figure}[t!]
\centering
\includegraphics[width=\linewidth]{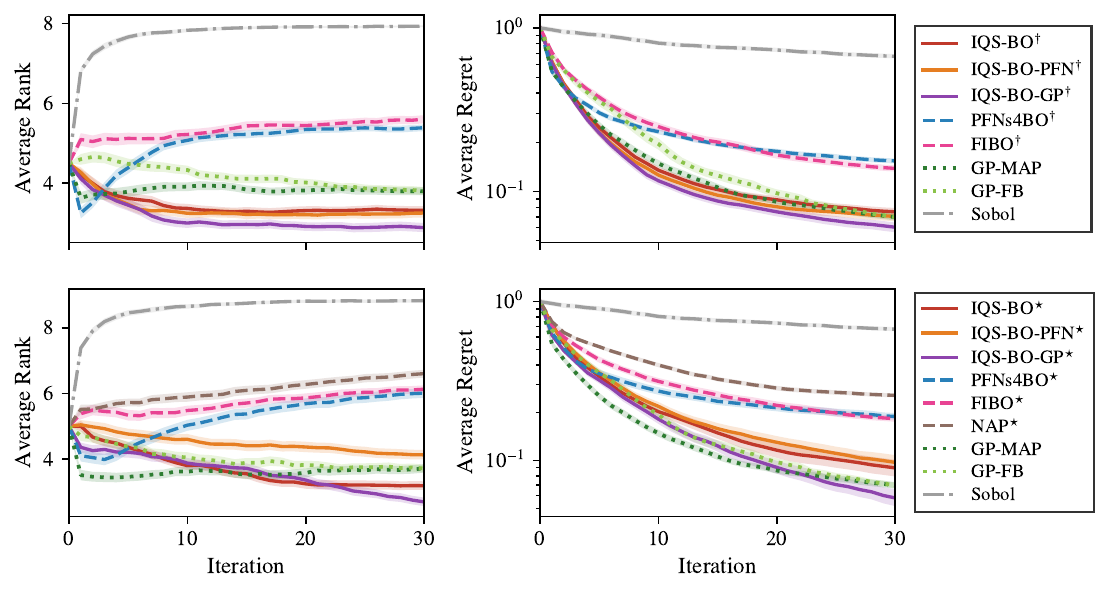}
\caption{As Figure~\ref{fig:main}, on the five 4D problems, $n_{\text{init}}=40$. Curves for $n_{\text{init}}=10$ and $20$ show the same pattern.}
\label{fig:4d}
\end{figure}

\begin{table}[t!]
\centering\small
\setlength{\abovecaptionskip}{2pt}
\setlength{\belowcaptionskip}{4pt}
\caption{Mean normalised final regret per 4D problem, $n_{\text{init}}=40$, 90 seeds, $\pm$ $95\%$ CI. The GP-prior column averages ten instances. Bold marks the best value per problem.}
\label{tab:4d-final}
\begin{tabular}{@{}lccccc@{}}
\toprule
Method & Ackley & Levy & Rosenbrock & Benzylation & GP prior \\
\midrule
\ours{}$^{\dagger}$ & $\mathbf{0.111}{\scriptstyle\pm0.006}$ & $0.016{\scriptstyle\pm0.008}$ & $0.028{\scriptstyle\pm0.002}$ & $0.141{\scriptstyle\pm0.016}$ & $0.079{\scriptstyle\pm0.011}$ \\
\ourspfn{}$^{\dagger}$ & $0.118{\scriptstyle\pm0.006}$ & $0.031{\scriptstyle\pm0.013}$ & $0.026{\scriptstyle\pm0.002}$ & $0.135{\scriptstyle\pm0.016}$ & $0.043{\scriptstyle\pm0.007}$ \\
\oursgp{}$^{\dagger}$ & $0.123{\scriptstyle\pm0.006}$ & $0.024{\scriptstyle\pm0.009}$ & $\mathbf{0.025}{\scriptstyle\pm0.002}$ & $0.113{\scriptstyle\pm0.018}$ & $0.018{\scriptstyle\pm0.004}$ \\
PFNs4BO$^{\dagger}$ & $0.432{\scriptstyle\pm0.020}$ & $0.116{\scriptstyle\pm0.016}$ & $0.080{\scriptstyle\pm0.008}$ & $\mathbf{0.000}{\scriptstyle\pm0.000}$ & $0.142{\scriptstyle\pm0.008}$ \\
\fibo{}$^{\dagger}$ & $0.295{\scriptstyle\pm0.014}$ & $0.138{\scriptstyle\pm0.025}$ & $0.059{\scriptstyle\pm0.006}$ & $0.049{\scriptstyle\pm0.005}$ & $0.152{\scriptstyle\pm0.009}$ \\
\midrule
\ours{}$^{\star}$ & $0.228{\scriptstyle\pm0.048}$ & $\mathbf{0.006}{\scriptstyle\pm0.003}$ & $0.038{\scriptstyle\pm0.003}$ & $0.102{\scriptstyle\pm0.018}$ & $0.076{\scriptstyle\pm0.010}$ \\
\ourspfn{}$^{\star}$ & $0.203{\scriptstyle\pm0.047}$ & $0.053{\scriptstyle\pm0.009}$ & $0.059{\scriptstyle\pm0.005}$ & $0.119{\scriptstyle\pm0.017}$ & $0.053{\scriptstyle\pm0.007}$ \\
\oursgp{}$^{\star}$ & $0.150{\scriptstyle\pm0.029}$ & $0.047{\scriptstyle\pm0.008}$ & $0.033{\scriptstyle\pm0.003}$ & $0.048{\scriptstyle\pm0.016}$ & $0.013{\scriptstyle\pm0.003}$ \\
PFNs4BO$^{\star}$ & $0.358{\scriptstyle\pm0.018}$ & $0.119{\scriptstyle\pm0.035}$ & $0.078{\scriptstyle\pm0.011}$ & $0.179{\scriptstyle\pm0.016}$ & $0.213{\scriptstyle\pm0.014}$ \\
\fibo{}$^{\star}$ & $0.274{\scriptstyle\pm0.010}$ & $0.240{\scriptstyle\pm0.038}$ & $0.056{\scriptstyle\pm0.004}$ & $0.276{\scriptstyle\pm0.006}$ & $0.070{\scriptstyle\pm0.009}$ \\
NAP$^{\star}$ & $0.335{\scriptstyle\pm0.022}$ & $0.065{\scriptstyle\pm0.010}$ & $0.097{\scriptstyle\pm0.012}$ & $0.480{\scriptstyle\pm0.018}$ & $0.307{\scriptstyle\pm0.018}$ \\
\midrule
\gpmap{} & $0.208{\scriptstyle\pm0.010}$ & $0.093{\scriptstyle\pm0.012}$ & $0.042{\scriptstyle\pm0.004}$ & $\mathbf{0.000}{\scriptstyle\pm0.000}$ & $0.006{\scriptstyle\pm0.001}$ \\
\gpfb{} & $0.218{\scriptstyle\pm0.008}$ & $0.087{\scriptstyle\pm0.011}$ & $0.039{\scriptstyle\pm0.004}$ & $\mathbf{0.000}{\scriptstyle\pm0.000}$ & $\mathbf{0.005}{\scriptstyle\pm0.001}$ \\
Sobol & $0.906{\scriptstyle\pm0.028}$ & $0.627{\scriptstyle\pm0.059}$ & $0.516{\scriptstyle\pm0.063}$ & $0.748{\scriptstyle\pm0.038}$ & $0.557{\scriptstyle\pm0.018}$ \\
\bottomrule
\end{tabular}

\end{table}

\begin{figure}[t!]
\centering
\includegraphics[width=\linewidth]{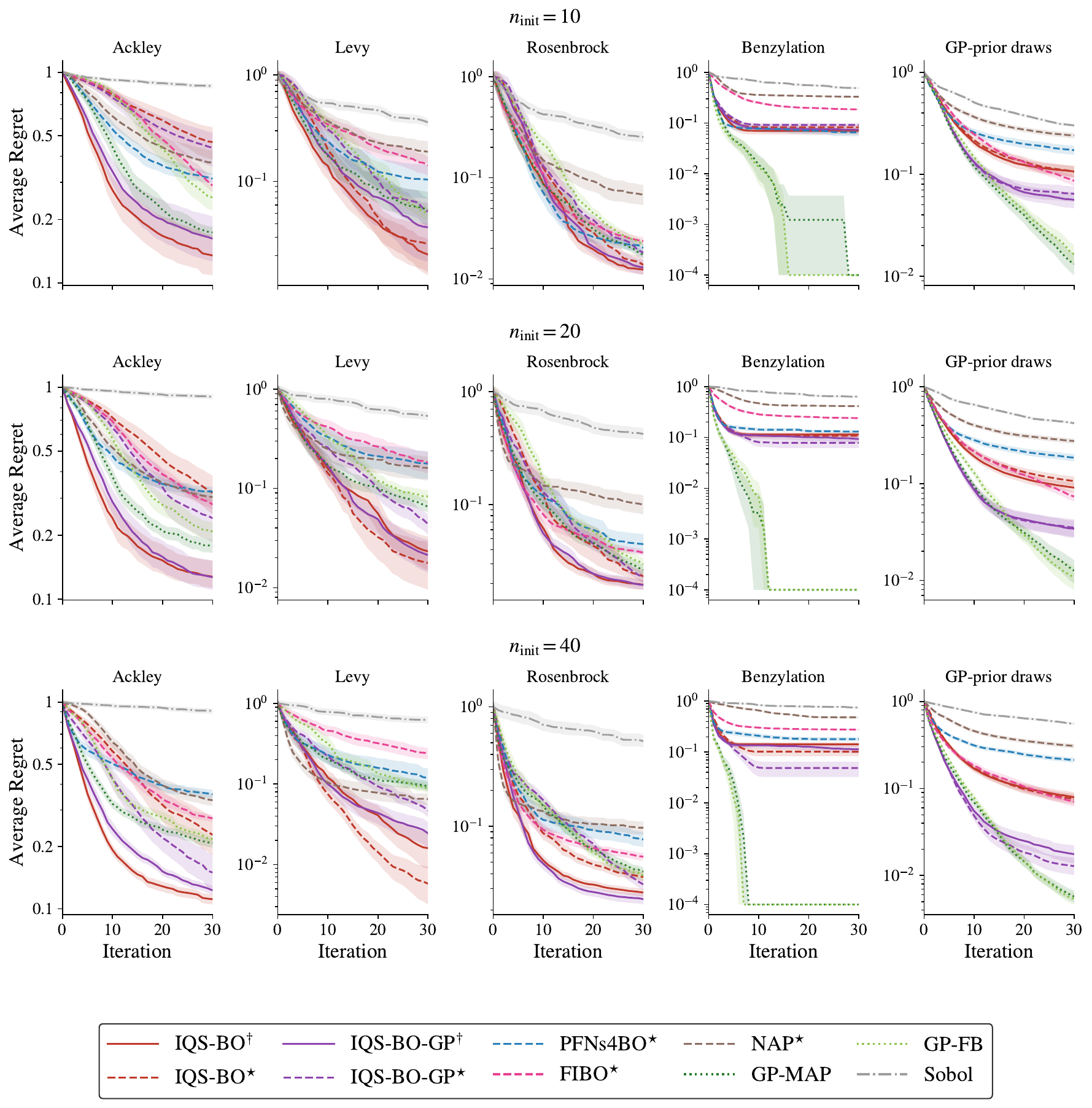}
\caption{Average regret per iteration on each 4D problem (columns), $n_{\text{init}}=10$, $20$ and $40$ (rows), 90 seeds (900 runs for the ten GP-prior draws), with $95\%$ confidence intervals.}
\label{fig:perproblem-4d}
\end{figure}

\subsection{Learned Query Selection Against Expected Improvement}
\label{app:controlled}
\label{app:GPACSBO}
\label{app:SFACSBO}

Table~\ref{tab:controlled} provides the paired win rates behind Section~\ref{sec:results-controlled} and adds two comparisons.

\paragraph{Dependence on the prior}
Pre-trained on the HEBO prior instead of the mixture prior, \oursgp{} beats \ctrl{} in $54$ to $58\%$ of runs on $R^{\text{final}}$ and loses to it on $R^{\text{AUC}}$, winning $30$ to $47\%$ (bottom block of Table~\ref{tab:controlled}). The learned rule still reaches a better final regret than expected improvement but is slower early in the run, and the clear advantage of Section~\ref{sec:results-controlled} comes from the mixture prior.

\paragraph{PFN surrogate}
The same comparison can be read from existing results with a PFN surrogate, since \ourspfn{} reads the released \pfnhebo{} model unchanged and differs from it only in how the query is chosen. \ourspfn{} ends with lower regret than \pfnhebo{} in $60$, $67$ and $75\%$ of runs at $n_{\text{init}}=10$, $20$ and $40$ (Table~\ref{tab:winrates3d}), and its gain is largest on the problems where expected improvement on the PFN posterior does worst against \gpmap{}.

\begin{table}[t!]
\centering\small
\caption{Paired win rate (\%) of the first method against the second, ten 3D problems, 900 paired runs per cell, $\pm$ $95\%$ bootstrap CI half-width. $^{\dagger}$ mixture prior, $^{\star}$ HEBO prior. \ctrl{} requires no pre-training, so its comparison with \gpmap{} is the same under both priors and is shown once.}
\label{tab:controlled}
\vspace{4pt}
\begin{tabular}{@{}lcccccc@{}}
\toprule
 & \multicolumn{3}{c}{$R^{\text{final}}$} & \multicolumn{3}{c}{$R^{\text{AUC}}$} \\
\cmidrule(lr){2-4}\cmidrule(l){5-7}
$n_{\text{init}}$ & 10 & 20 & 40 & 10 & 20 & 40 \\
\midrule
\oursgp{}$^{\dagger}$ vs \ctrl{} & $66.1{\scriptstyle\pm3.1}$ & $68.2{\scriptstyle\pm3.1}$ & $70.3{\scriptstyle\pm2.9}$ & $53.3{\scriptstyle\pm3.2}$ & $61.6{\scriptstyle\pm3.2}$ & $69.6{\scriptstyle\pm3.0}$ \\
\ctrl{} vs \gpmap{} & $70.7{\scriptstyle\pm3.1}$ & $71.2{\scriptstyle\pm2.9}$ & $67.3{\scriptstyle\pm3.1}$ & $61.7{\scriptstyle\pm3.2}$ & $61.9{\scriptstyle\pm3.2}$ & $62.5{\scriptstyle\pm3.1}$ \\
\oursgp{}$^{\dagger}$ vs \gpmap{} & $77.6{\scriptstyle\pm2.7}$ & $78.3{\scriptstyle\pm2.7}$ & $81.1{\scriptstyle\pm2.6}$ & $62.0{\scriptstyle\pm3.2}$ & $70.4{\scriptstyle\pm3.0}$ & $75.1{\scriptstyle\pm2.8}$ \\
\midrule
\oursgp{}$^{\star}$ vs \ctrl{} & $54.3{\scriptstyle\pm3.3}$ & $57.7{\scriptstyle\pm3.3}$ & $56.2{\scriptstyle\pm3.3}$ & $30.1{\scriptstyle\pm3.0}$ & $37.2{\scriptstyle\pm3.2}$ & $46.8{\scriptstyle\pm3.2}$ \\
\oursgp{}$^{\star}$ vs \gpmap{} & $71.0{\scriptstyle\pm2.9}$ & $75.0{\scriptstyle\pm2.8}$ & $70.1{\scriptstyle\pm3.0}$ & $38.8{\scriptstyle\pm3.2}$ & $49.0{\scriptstyle\pm3.3}$ & $58.7{\scriptstyle\pm3.2}$ \\
\bottomrule
\end{tabular}
\end{table}

\subsection{Effect of the Pre-Training Prior}
\label{app:prior-results}
\label{app:prior-ablation}

Table~\ref{tab:prior-winrates} provides the paired win rates behind the prior comparison of Section~\ref{sec:results-controlled} and adds two analyses.

\paragraph{Effect over the run}
For all three \ours{} variants, the win rates of the mixture prior on $R^{\text{AUC}}$ exceed those on $R^{\text{final}}$ (Table~\ref{tab:prior-winrates}), so the prior helps most early in the run. The regret curves of Figure~\ref{fig:prior-curves} show the same, as the gap between the two priors opens in the first rounds, when the model relies more on what it learned in pre-training than on the observations, and then persists.

\paragraph{Problems with well-specified GP surrogates}
Table~\ref{tab:prior-perproblem} tests whether adding non-GP functions to the prior hurts on problems where a GP is already the right model. For \ours{} and \ourspfn{}, the mixture-prior versions win about half of the runs or more on Rosenbrock and the GP-prior draws, so the broader prior costs them nothing there. On benzylation, both priors often reach the optimum, so ties pull the win rates towards $50\%$ at small context sizes, but at $n_{\text{init}}=40$ the HEBO-prior versions of \ours{} and \oursgp{} win $61\%$ and $68\%$ of paired runs. Elsewhere, the mixture prior lowers performance only for \oursgp{} on the 3D Rosenbrock, PFNs4BO on the 4D Rosenbrock and \fibo{} on the GP-prior draws.
\begin{table}[htbp]
\centering\small
\caption{Paired win rate (\%) of each model pre-trained on the mixture prior against the same model pre-trained on its original prior, ten 3D problems, 900 paired runs per cell, $\pm$ $95\%$ bootstrap CI half-width.}
\label{tab:prior-winrates}
\vspace{4pt}
\begin{tabular}{@{}lcccccc@{}}
\toprule
 & \multicolumn{3}{c}{$R^{\text{final}}$} & \multicolumn{3}{c}{$R^{\text{AUC}}$} \\
\cmidrule(lr){2-4}\cmidrule(l){5-7}
$n_{\text{init}}$ & 10 & 20 & 40 & 10 & 20 & 40 \\
\midrule
\ours{}    & $66.7{\scriptstyle\pm3.1}$ & $68.3{\scriptstyle\pm3.1}$ & $68.8{\scriptstyle\pm3.1}$ & $76.1{\scriptstyle\pm2.8}$ & $75.7{\scriptstyle\pm2.8}$ & $76.7{\scriptstyle\pm2.7}$ \\
\ourspfn{} & $77.5{\scriptstyle\pm2.8}$ & $76.0{\scriptstyle\pm2.8}$ & $76.2{\scriptstyle\pm2.8}$ & $80.7{\scriptstyle\pm2.6}$ & $82.6{\scriptstyle\pm2.5}$ & $83.4{\scriptstyle\pm2.3}$ \\
\oursgp{}  & $62.2{\scriptstyle\pm3.2}$ & $61.8{\scriptstyle\pm3.1}$ & $65.8{\scriptstyle\pm3.1}$ & $70.3{\scriptstyle\pm3.0}$ & $70.9{\scriptstyle\pm2.9}$ & $68.7{\scriptstyle\pm3.1}$ \\
PFNs4BO    & $53.0{\scriptstyle\pm3.3}$ & $55.5{\scriptstyle\pm3.2}$ & $62.9{\scriptstyle\pm3.2}$ & $56.7{\scriptstyle\pm3.2}$ & $54.3{\scriptstyle\pm3.2}$ & $63.0{\scriptstyle\pm3.2}$ \\
\fibo{}    & $51.1{\scriptstyle\pm3.3}$ & $47.9{\scriptstyle\pm3.3}$ & $50.9{\scriptstyle\pm3.3}$ & $50.9{\scriptstyle\pm3.2}$ & $53.2{\scriptstyle\pm3.3}$ & $54.2{\scriptstyle\pm3.2}$ \\
\bottomrule
\end{tabular}
\end{table}

\begin{table}[htbp]
\centering\footnotesize
\setlength{\tabcolsep}{2pt}
\setlength{\abovecaptionskip}{2pt}
\setlength{\belowcaptionskip}{6pt}
\caption{Paired win rate (\%) on $R^{\text{final}}$ of each model pre-trained on the mixture prior against the same model pre-trained on its original prior, on the problems for which GP surrogates are well specified, $\pm$ $95\%$ CI. 90 paired runs per cell (900 for the GP draws).}
\label{tab:prior-perproblem}
\begin{tabular*}{\linewidth}{@{}l@{\extracolsep{\fill}}ccc@{\hspace{8pt}}ccc@{\hspace{8pt}}ccc@{\hspace{8pt}}ccc@{}}
\toprule
 & \multicolumn{3}{c}{Rosenbrock (3D)} & \multicolumn{3}{c}{Rosenbrock (4D)} & \multicolumn{3}{c}{GP draws (4D)} & \multicolumn{3}{c}{Benzylation (4D)} \\
\cmidrule(lr){2-4}\cmidrule(lr){5-7}\cmidrule(lr){8-10}\cmidrule(l){11-13}
$n_{\text{init}}$ & 10 & 20 & 40 & 10 & 20 & 40 & 10 & 20 & 40 & 10 & 20 & 40 \\
\midrule
\ours{} & $62{\scriptstyle\pm10}$ & $58{\scriptstyle\pm10}$ & $64{\scriptstyle\pm10}$ & $58{\scriptstyle\pm10}$ & $67{\scriptstyle\pm9}$ & $78{\scriptstyle\pm8}$ & $55{\scriptstyle\pm3}$ & $53{\scriptstyle\pm3}$ & $46{\scriptstyle\pm3}$ & $56{\scriptstyle\pm6}$ & $48{\scriptstyle\pm6}$ & $39{\scriptstyle\pm7}$ \\
\ourspfn{} & $72{\scriptstyle\pm9}$ & $64{\scriptstyle\pm10}$ & $61{\scriptstyle\pm10}$ & $80{\scriptstyle\pm8}$ & $92{\scriptstyle\pm5}$ & $91{\scriptstyle\pm6}$ & $63{\scriptstyle\pm3}$ & $65{\scriptstyle\pm3}$ & $65{\scriptstyle\pm3}$ & $45{\scriptstyle\pm6}$ & $45{\scriptstyle\pm6}$ & $45{\scriptstyle\pm7}$ \\
\oursgp{} & $44{\scriptstyle\pm10}$ & $26{\scriptstyle\pm9}$ & $51{\scriptstyle\pm10}$ & $69{\scriptstyle\pm9}$ & $66{\scriptstyle\pm10}$ & $61{\scriptstyle\pm10}$ & $46{\scriptstyle\pm3}$ & $44{\scriptstyle\pm3}$ & $48{\scriptstyle\pm3}$ & $58{\scriptstyle\pm6}$ & $45{\scriptstyle\pm7}$ & $32{\scriptstyle\pm6}$ \\
PFNs4BO & $49{\scriptstyle\pm10}$ & $56{\scriptstyle\pm10}$ & $74{\scriptstyle\pm9}$ & $26{\scriptstyle\pm9}$ & $28{\scriptstyle\pm9}$ & $40{\scriptstyle\pm10}$ & $49{\scriptstyle\pm3}$ & $58{\scriptstyle\pm3}$ & $59{\scriptstyle\pm3}$ & $78{\scriptstyle\pm6}$ & $95{\scriptstyle\pm3}$ & $97{\scriptstyle\pm2}$ \\
\fibo{} & $50{\scriptstyle\pm10}$ & $44{\scriptstyle\pm10}$ & $46{\scriptstyle\pm10}$ & $48{\scriptstyle\pm10}$ & $60{\scriptstyle\pm10}$ & $51{\scriptstyle\pm11}$ & $20{\scriptstyle\pm2}$ & $16{\scriptstyle\pm2}$ & $13{\scriptstyle\pm2}$ & $100{\scriptstyle\pm0}$ & $100{\scriptstyle\pm0}$ & $100{\scriptstyle\pm0}$ \\
\bottomrule
\end{tabular*}
\end{table}

\begin{figure}[t!]
\centering
\includegraphics[width=\linewidth]{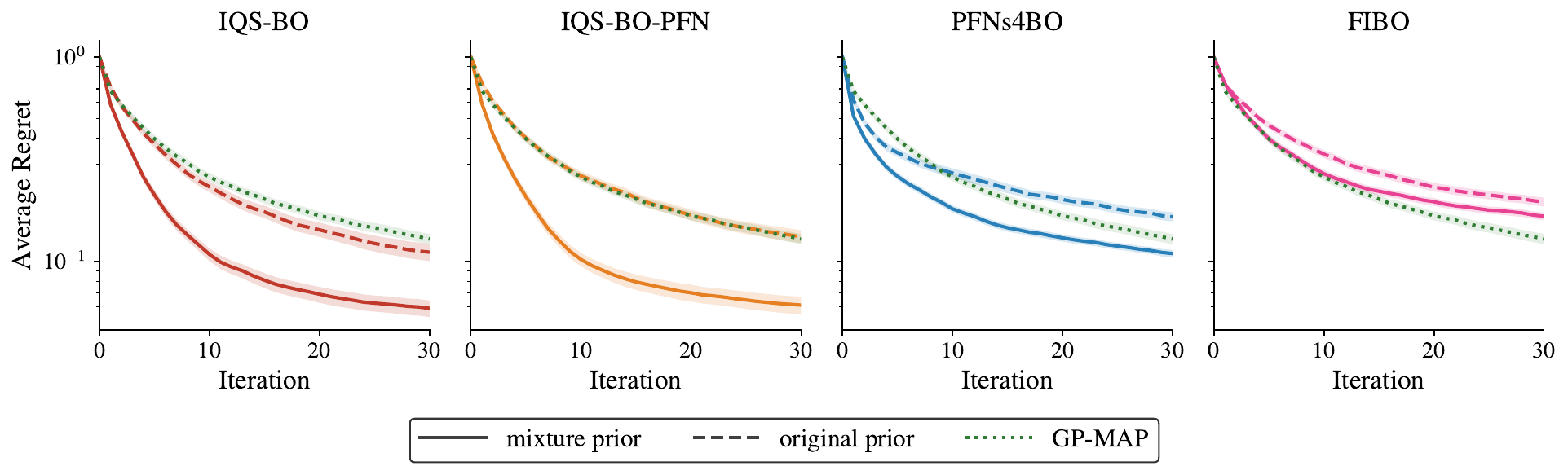}
\caption{Average regret per iteration of \ours{}, \ourspfn{}, PFNs4BO and \fibo{} pre-trained on the mixture prior (solid) and on their original prior (dashed), with \gpmap{} as reference, ten 3D problems, $n_{\text{init}}=40$, 90 seeds, with $95\%$ confidence intervals.}
\label{fig:prior-curves}
\end{figure}

\subsection{Batch Bayesian Optimisation}
\label{app:bobatch}

Table~\ref{tab:batch-perproblem} provides the per-problem GAP values behind Figure~\ref{fig:batch}, for the seven 3D problems of the batch study, all but Camel and the two Olympus campaigns (Appendix~\ref{app:benchmarks}). We provide two complementary analyses.

\paragraph{Per-problem values}
Figure~\ref{fig:abl-batch-perproblem} shows how the GAP of \ours{} and \ourspfn{} evolves over the evaluations on each problem. The mean GAP of \ours{} stays between $0.91$ and $0.93$ across batch sizes, and its margin to the best baseline widens from $0.05$ at $q=10$ to $0.08$ at $q=50$. The batch-size ablation attributes this robustness to pre-training at the deployed batch size (Appendix~\ref{app:ablations}).

\paragraph{Proposal time per batch}
Table~\ref{tab:cost} complements the run-level proposal times of Section~\ref{sec:results-batch} with the time of a single proposal on a GPU workstation, for batch sizes from $1$ to $20$. The cost of every \ours{} variant grows slowly with $q$, as it needs one forward pass per point, whereas that of \gpmap{} and \pfnhebo{} grows steeply. At $q=20$, \ours{} and \ourspfn{} are more than an order of magnitude faster than \gpmap{}, and only \fibo{} is faster, because it draws the points of a batch independently, without conditioning on earlier draws.

\begin{table}[!htb]
\centering\footnotesize
\setlength{\tabcolsep}{1pt}
\renewcommand{\arraystretch}{0.95}
\setlength{\abovecaptionskip}{2pt}
\setlength{\belowcaptionskip}{2pt}
\caption{GAP after 200 evaluations per problem and mean over the seven problems, mean over 90 seeds, with the $95\%$ CI half-width in units of $10^{-3}$ as subscript, and seconds spent proposing queries over a complete run. Bold marks the best value per problem and batch size.}
\label{tab:batch-perproblem}
\begin{tabular}{@{}cl*{8}{>{\centering\arraybackslash}p{0.083\linewidth}}>{\centering\arraybackslash}p{0.055\linewidth}@{}}
\toprule
 & method & Ackley & Hartm. & Bioproc. & Levy & Michal. & Rastr. & Rosenb. & mean & s/run \\
\midrule
\multirow{7}{*}{\rotatebox[origin=c]{90}{$q=10$}} & \ours{} & $.940{\scriptscriptstyle\pm6}$ & $\mathbf{.999{\scriptscriptstyle\pm0}}$ & $\mathbf{.999{\scriptscriptstyle\pm0}}$ & $\mathbf{.998{\scriptscriptstyle\pm1}}$ & $.853{\scriptscriptstyle\pm51}$ & $\mathbf{.750{\scriptscriptstyle\pm39}}$ & $.909{\scriptscriptstyle\pm20}$ & $\mathbf{.921{\scriptscriptstyle\pm10}}$ & $0.9$ \\
 & \ourspfn{} & $\mathbf{.949{\scriptscriptstyle\pm6}}$ & $.998{\scriptscriptstyle\pm1}$ & $.998{\scriptscriptstyle\pm1}$ & $.954{\scriptscriptstyle\pm13}$ & $\mathbf{.874{\scriptscriptstyle\pm47}}$ & $.741{\scriptscriptstyle\pm38}$ & $.933{\scriptscriptstyle\pm30}$ & $\mathbf{.921{\scriptscriptstyle\pm10}}$ & $3.2$ \\
 & \gpmap{} & $.839{\scriptscriptstyle\pm13}$ & $\mathbf{.999{\scriptscriptstyle\pm0}}$ & $.995{\scriptscriptstyle\pm2}$ & $.857{\scriptscriptstyle\pm35}$ & $.760{\scriptscriptstyle\pm57}$ & $.682{\scriptscriptstyle\pm40}$ & $\mathbf{.973{\scriptscriptstyle\pm6}}$ & $.872{\scriptscriptstyle\pm11}$ & $54.5$ \\
 & \gpfb{} & $.834{\scriptscriptstyle\pm14}$ & $\mathbf{.999{\scriptscriptstyle\pm0}}$ & $.994{\scriptscriptstyle\pm2}$ & $.885{\scriptscriptstyle\pm29}$ & $.719{\scriptscriptstyle\pm61}$ & $.679{\scriptscriptstyle\pm39}$ & $.971{\scriptscriptstyle\pm11}$ & $.869{\scriptscriptstyle\pm11}$ & $350.3$ \\
 & \pfnhebo{} & $.773{\scriptscriptstyle\pm20}$ & $.968{\scriptscriptstyle\pm10}$ & $.879{\scriptscriptstyle\pm41}$ & $.946{\scriptscriptstyle\pm16}$ & $.797{\scriptscriptstyle\pm39}$ & $.714{\scriptscriptstyle\pm31}$ & $.855{\scriptscriptstyle\pm37}$ & $.848{\scriptscriptstyle\pm11}$ & $182.6$ \\
 & \pfnbnn{} & $.719{\scriptscriptstyle\pm18}$ & $.975{\scriptscriptstyle\pm11}$ & $.987{\scriptscriptstyle\pm8}$ & $.949{\scriptscriptstyle\pm15}$ & $.547{\scriptscriptstyle\pm59}$ & $.678{\scriptscriptstyle\pm35}$ & $.850{\scriptscriptstyle\pm39}$ & $.815{\scriptscriptstyle\pm12}$ & $131.8$ \\
 & \fibo{} & $.805{\scriptscriptstyle\pm15}$ & $.994{\scriptscriptstyle\pm2}$ & $.975{\scriptscriptstyle\pm14}$ & $.941{\scriptscriptstyle\pm16}$ & $.810{\scriptscriptstyle\pm39}$ & $.568{\scriptscriptstyle\pm43}$ & $.863{\scriptscriptstyle\pm35}$ & $.851{\scriptscriptstyle\pm10}$ & $0.6$ \\
\midrule
\multirow{7}{*}{\rotatebox[origin=c]{90}{$q=20$}} & \ours{} & $.930{\scriptscriptstyle\pm10}$ & $\mathbf{.998{\scriptscriptstyle\pm1}}$ & $\mathbf{.999{\scriptscriptstyle\pm1}}$ & $\mathbf{.995{\scriptscriptstyle\pm3}}$ & $.871{\scriptscriptstyle\pm40}$ & $.781{\scriptscriptstyle\pm37}$ & $.931{\scriptscriptstyle\pm18}$ & $.929{\scriptscriptstyle\pm8}$ & $0.9$ \\
 & \ourspfn{} & $\mathbf{.947{\scriptscriptstyle\pm6}}$ & $\mathbf{.998{\scriptscriptstyle\pm1}}$ & $.998{\scriptscriptstyle\pm1}$ & $.922{\scriptscriptstyle\pm25}$ & $\mathbf{.953{\scriptscriptstyle\pm22}}$ & $\mathbf{.790{\scriptscriptstyle\pm31}}$ & $.937{\scriptscriptstyle\pm25}$ & $\mathbf{.935{\scriptscriptstyle\pm7}}$ & $3.1$ \\
 & \gpmap{} & $.819{\scriptscriptstyle\pm13}$ & $\mathbf{.998{\scriptscriptstyle\pm0}}$ & $.983{\scriptscriptstyle\pm7}$ & $.832{\scriptscriptstyle\pm41}$ & $.680{\scriptscriptstyle\pm55}$ & $.598{\scriptscriptstyle\pm49}$ & $\mathbf{.964{\scriptscriptstyle\pm10}}$ & $.839{\scriptscriptstyle\pm12}$ & $60.8$ \\
 & \gpfb{} & $.816{\scriptscriptstyle\pm12}$ & $\mathbf{.998{\scriptscriptstyle\pm0}}$ & $.976{\scriptscriptstyle\pm9}$ & $.808{\scriptscriptstyle\pm48}$ & $.631{\scriptscriptstyle\pm64}$ & $.617{\scriptscriptstyle\pm43}$ & $.960{\scriptscriptstyle\pm14}$ & $.829{\scriptscriptstyle\pm13}$ & $417.0$ \\
 & \pfnhebo{} & $.756{\scriptscriptstyle\pm15}$ & $.967{\scriptscriptstyle\pm8}$ & $.859{\scriptscriptstyle\pm47}$ & $.943{\scriptscriptstyle\pm15}$ & $.809{\scriptscriptstyle\pm38}$ & $.715{\scriptscriptstyle\pm37}$ & $.852{\scriptscriptstyle\pm33}$ & $.843{\scriptscriptstyle\pm12}$ & $208.5$ \\
 & \pfnbnn{} & $.710{\scriptscriptstyle\pm21}$ & $.974{\scriptscriptstyle\pm9}$ & $.989{\scriptscriptstyle\pm6}$ & $.951{\scriptscriptstyle\pm13}$ & $.589{\scriptscriptstyle\pm60}$ & $.664{\scriptscriptstyle\pm37}$ & $.878{\scriptscriptstyle\pm29}$ & $.822{\scriptscriptstyle\pm12}$ & $135.8$ \\
 & \fibo{} & $.806{\scriptscriptstyle\pm17}$ & $.993{\scriptscriptstyle\pm2}$ & $.965{\scriptscriptstyle\pm18}$ & $.935{\scriptscriptstyle\pm23}$ & $.768{\scriptscriptstyle\pm48}$ & $.557{\scriptscriptstyle\pm46}$ & $.864{\scriptscriptstyle\pm34}$ & $.841{\scriptscriptstyle\pm12}$ & $0.4$ \\
\midrule
\multirow{7}{*}{\rotatebox[origin=c]{90}{$q=50$}} & \ours{} & $.860{\scriptscriptstyle\pm13}$ & $.998{\scriptscriptstyle\pm1}$ & $.989{\scriptscriptstyle\pm5}$ & $\mathbf{.975{\scriptscriptstyle\pm7}}$ & $.853{\scriptscriptstyle\pm46}$ & $\mathbf{.796{\scriptscriptstyle\pm27}}$ & $.892{\scriptscriptstyle\pm34}$ & $.909{\scriptscriptstyle\pm9}$ & $1.2$ \\
 & \ourspfn{} & $\mathbf{.917{\scriptscriptstyle\pm11}}$ & $\mathbf{.999{\scriptscriptstyle\pm0}}$ & $\mathbf{.992{\scriptscriptstyle\pm4}}$ & $.909{\scriptscriptstyle\pm25}$ & $\mathbf{.933{\scriptscriptstyle\pm25}}$ & $.783{\scriptscriptstyle\pm29}$ & $\mathbf{.952{\scriptscriptstyle\pm19}}$ & $\mathbf{.926{\scriptscriptstyle\pm7}}$ & $3.9$ \\
 & \gpmap{} & $.787{\scriptscriptstyle\pm13}$ & $.998{\scriptscriptstyle\pm1}$ & $.945{\scriptscriptstyle\pm17}$ & $.741{\scriptscriptstyle\pm52}$ & $.619{\scriptscriptstyle\pm58}$ & $.550{\scriptscriptstyle\pm50}$ & $.931{\scriptscriptstyle\pm25}$ & $.796{\scriptscriptstyle\pm14}$ & $45.1$ \\
 & \gpfb{} & $.799{\scriptscriptstyle\pm16}$ & $.998{\scriptscriptstyle\pm0}$ & $.954{\scriptscriptstyle\pm11}$ & $.703{\scriptscriptstyle\pm61}$ & $.489{\scriptscriptstyle\pm65}$ & $.599{\scriptscriptstyle\pm52}$ & $.926{\scriptscriptstyle\pm27}$ & $.781{\scriptscriptstyle\pm15}$ & $396.3$ \\
 & \pfnhebo{} & $.692{\scriptscriptstyle\pm21}$ & $.953{\scriptscriptstyle\pm16}$ & $.774{\scriptscriptstyle\pm52}$ & $.908{\scriptscriptstyle\pm28}$ & $.781{\scriptscriptstyle\pm44}$ & $.604{\scriptscriptstyle\pm49}$ & $.780{\scriptscriptstyle\pm51}$ & $.784{\scriptscriptstyle\pm15}$ & $225.4$ \\
 & \pfnbnn{} & $.667{\scriptscriptstyle\pm17}$ & $.971{\scriptscriptstyle\pm9}$ & $.986{\scriptscriptstyle\pm5}$ & $.946{\scriptscriptstyle\pm14}$ & $.657{\scriptscriptstyle\pm50}$ & $.641{\scriptscriptstyle\pm39}$ & $.858{\scriptscriptstyle\pm35}$ & $.818{\scriptscriptstyle\pm11}$ & $98.7$ \\
 & \fibo{} & $.806{\scriptscriptstyle\pm14}$ & $.988{\scriptscriptstyle\pm4}$ & $.955{\scriptscriptstyle\pm26}$ & $.911{\scriptscriptstyle\pm32}$ & $.755{\scriptscriptstyle\pm48}$ & $.538{\scriptscriptstyle\pm48}$ & $.852{\scriptscriptstyle\pm33}$ & $.829{\scriptscriptstyle\pm12}$ & $0.3$ \\
\bottomrule
\end{tabular}

\end{table}

\begin{table}[!htb]
\centering\small
\setlength{\abovecaptionskip}{2pt}
\caption{Median seconds per proposed batch on a GPU workstation (Quadro P2200, Xeon W-1350P), Ackley in 3D, $n=30$ observations, ten repeats (three for \pfnhebo{}), with the interquartile range in brackets. $^{\dagger}$ mixture prior.}
\label{tab:cost}
\begin{tabular}{@{}l@{\hspace{6pt}}c@{\hspace{8pt}}c@{\hspace{8pt}}c@{\hspace{8pt}}c@{}}
\toprule
$q$ & 1 & 2 & 10 & 20 \\
\midrule
\ours{}$^{\dagger}$ & 0.011\,{\footnotesize[0.011, 0.011]} & 0.020\,{\footnotesize[0.020, 0.022]} & 0.085\,{\footnotesize[0.085, 0.085]} & 0.170\,{\footnotesize[0.170, 0.171]} \\
\ourspfn{}$^{\dagger}$ & 0.030\,{\footnotesize[0.030, 0.030]} & 0.060\,{\footnotesize[0.060, 0.060]} & 0.302\,{\footnotesize[0.301, 0.303]} & 0.610\,{\footnotesize[0.604, 0.612]} \\
\oursgp{}$^{\dagger}$ & 0.096\,{\footnotesize[0.096, 0.096]} & 0.105\,{\footnotesize[0.105, 0.106]} & 0.185\,{\footnotesize[0.183, 0.186]} & 0.282\,{\footnotesize[0.280, 0.284]} \\
\fibo{} & 0.021\,{\footnotesize[0.021, 0.021]} & 0.022\,{\footnotesize[0.022, 0.022]} & 0.022\,{\footnotesize[0.022, 0.022]} & 0.022\,{\footnotesize[0.022, 0.023]} \\
\gpmap{} & 0.233\,{\footnotesize[0.233, 0.233]} & 0.232\,{\footnotesize[0.232, 0.233]} & 3.47\,{\footnotesize[3.47, 3.47]} & 10.7\,{\footnotesize[10.7, 10.8]} \\
\pfnhebo{} & 3.18\,{\footnotesize[3.17, 3.19]} & 6.78\,{\footnotesize[6.76, 6.79]} & 32.6\,{\footnotesize[32.6, 32.6]} & 64.9\,{\footnotesize[64.9, 65.0]} \\
\bottomrule
\end{tabular}

\end{table}
\subsection{Additional Ablations}
\label{app:ablations}

This section reports ablation studies not included in Section~\ref{sec:ablation}.

\paragraph{Candidate composition with other scoring rules}
We repeat the candidate-set ablation of Section~\ref{sec:ablation} with \oursgp{} and with \ctrl{}, and compare each against \gpmap{}. With the default candidate set, which contains $30\%$ perturbations of the best observations, \ours{}, \oursgp{} and \ctrl{} beat \gpmap{} in $77$ to $80\%$, $78$ to $81\%$ and $67$ to $71\%$ of paired runs. With a space-filling set of only Sobol and uniform candidates, all three lose to it, winning only $43$ to $45\%$, $36$ to $40\%$ and $27$ to $33\%$ of runs. The perturbed candidates therefore matter regardless of the scoring rule, since they are the main way in which the models can exploit near the best observations.

\paragraph{Batch size and batch construction}
Figure~\ref{fig:abl-batch-perproblem} compares each checkpoint pre-trained at the deployed batch size with the $q=2$ checkpoint used at the same $q$. For \ours{}, the pre-trained checkpoint beats the $q=2$ checkpoint in $62$, $91$ and $98\%$ of paired runs at $q=10$, $20$ and $50$, and the mean GAP of the $q=2$ checkpoint drops from $0.88$ to $0.63$ and $0.44$, while that of the pre-trained checkpoint stays above $0.90$. For \ourspfn{}, the $q=2$ checkpoint degrades much less, with a mean GAP between $0.84$ and $0.90$, likely because the frozen PFN conditions on the pending points and so still separates them. Sequential selection also matters more as the batch grows. At $q=10$, it outperforms top-$q$ selection, which takes the $q$ highest-scoring candidates of a single forward pass without pending points, in $68\%$ of paired runs.

\begin{figure}[t!]
\centering
\includegraphics[width=\linewidth,height=0.93\textheight,keepaspectratio]{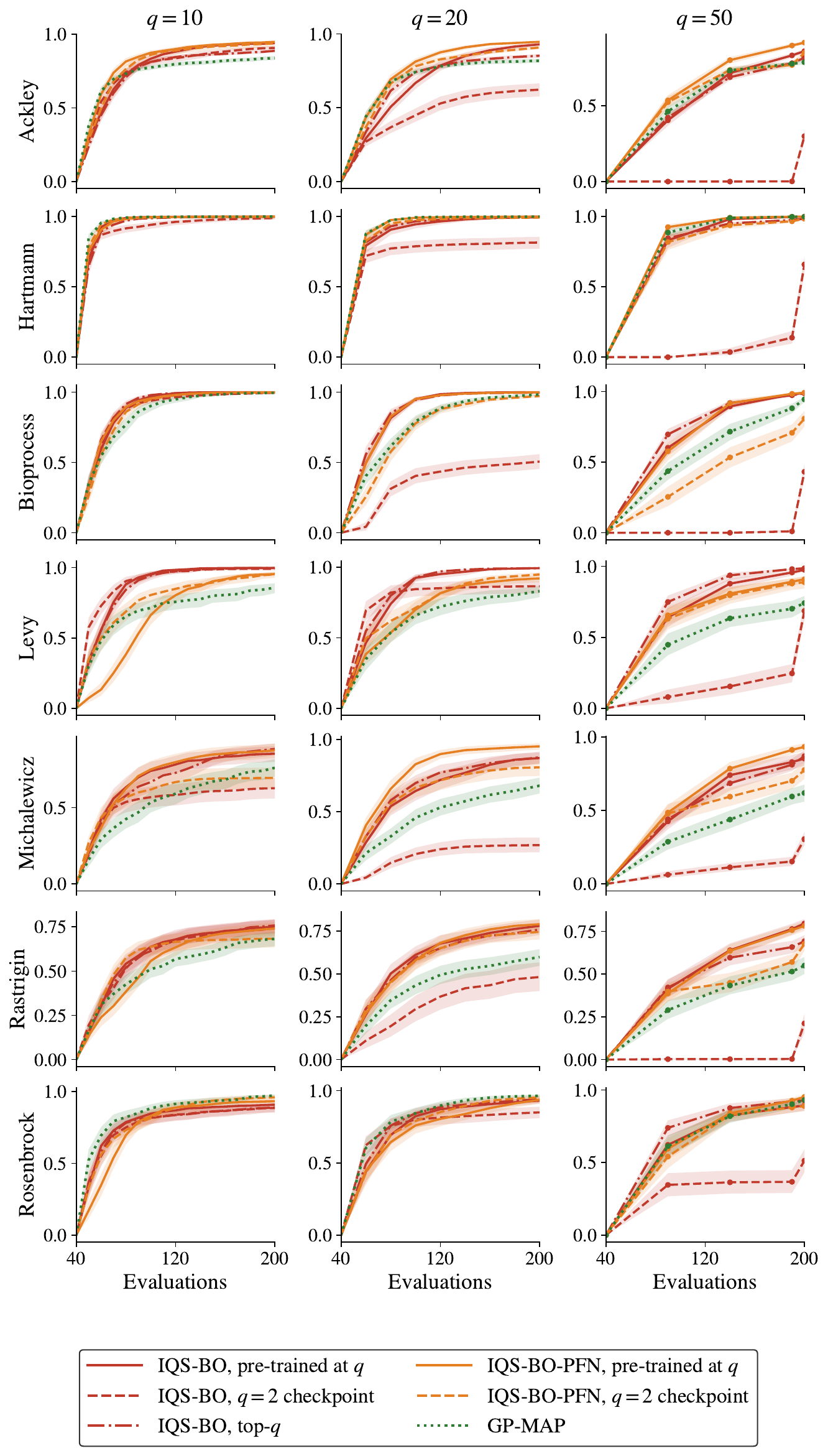}
\caption{GAP per evaluation on each problem of the batch study (rows), for $q=10$, $20$ and $50$ (columns), networks pre-trained on the HEBO prior, $n_{\text{init}}=40$, 90 seeds, with $95\%$ confidence intervals. The y-axis of every panel shows the GAP. Solid lines are checkpoints pre-trained at the deployed batch size, dashed lines the $q=2$ checkpoints used at that batch size.}
\label{fig:abl-batch-perproblem}
\end{figure}

\end{document}